\documentclass[12pt]{article}
\DeclareUnicodeCharacter{2212}{\textendash}
\usepackage{framed}
\usepackage{algorithm}
\usepackage[noend]{algpseudocode}
\usepackage[english]{babel}
\usepackage[utf8x]{inputenc}
\usepackage[T1]{fontenc}
\usepackage{apacite}
\usepackage{dirtytalk}
\usepackage{algorithm}
\usepackage{tikz}
\usetikzlibrary{positioning}
\usepackage{booktabs}
\usepackage{subfiles} 
\usepackage{multirow}
\usepackage{siunitx}
\usepackage{algpseudocode}
\usepackage{mathrsfs}
\usepackage{listings}
\usepackage{color}
\usepackage{bbm}
\usepackage{adjustbox}
\usepackage{epsfig, epsf, graphicx}
\usepackage{tikz}
\usepackage{animate}
\usepackage{pstricks, pst-node, psfrag}
\usepackage{amssymb, amsmath, amsthm, bm}
\usepackage{verbatim,enumerate}
\usepackage{rotating, lscape}
\usepackage{diagbox}
\usepackage[doublespacing]{setspace}

\usetikzlibrary{arrows,calc,tikzmark}
\usepackage[colorlinks=true,linkcolor=blue,citecolor=blue]{hyperref}%
\usepackage{float}
\usepackage{natbib}
\usepackage[small]{caption}
\usepackage{subcaption}
\usepackage{lipsum}
\usepackage{threeparttable}
\usepackage{multirow}
\usepackage{lmodern}
\usepackage{multicol}
\usepackage[toc,page]{appendix}
\graphicspath{{figures/}}

\makeatletter
\def\BState{\State\hskip-\ALG@thistlm}
\makeatother
\makeatletter
\newcounter{phase}[algorithm]
\newlength{\phaserulewidth}
\newcommand{\setphaserulewidth}{\setlength{\phaserulewidth}}

\makeatother
\setphaserulewidth{.7pt}
\algrenewcommand\algorithmicrequire{\textbf{Input:}}
\algrenewcommand\algorithmicensure{\textbf{Output:}}
\newcounter{case}[algorithm]
\newlength{\caserulewidth}
\newcommand{\setcaserulewidth}{\setlength{\caserulewidth}}

\makeatother
\setcaserulewidth{.7pt}
\algdef{SE}[SUBALG]{Indent}{EndIndent}{}{\algorithmicend\ }%
\algtext*{Indent}
\algtext*{EndIndent}

\def\diag{\hbox{diag}}

\def\diag{\hbox{diag}}

\def\boxit#1{\vbox{\hrule\hbox{\vrule\kern6pt
			\vbox{\kern6pt#1\kern6pt}\kern6pt\vrule}\hrule}}

\def\bse{\begin{eqnarray*}}
	\def\ese{\end{eqnarray*}}
\def\be{\begin{eqnarray}}
\def\ee{\end{eqnarray}}
\def\bq{\begin{equation}}
\def\eq{\end{equation}}
\def\bse{\begin{eqnarray*}}
	\def\ese{\end{eqnarray*}}

\newcommand{\0}{\mathbf{0}}

\newcommand{\bbeta}{\bm{\beta}}

\newcommand{\bGamma}{\bm{\Gamma}}
\newcommand{\bgamma}{\bm{\gamma}}
\newcommand{\bH}{\mathbf{H}}
\newcommand{\bI}{\mathbf{I}}

\newcommand{\bmu}{\bm{\mu}}
\newcommand{\bOmega}{\bm{\Omega}}
\newcommand{\bPsi}{\bm{\Psi}}

\newcommand{\bomega}{\bm{\omega}}
\newcommand{\bDelta}{\bm{\Delta}}
\newcommand{\bP}{\mathbf{P}}

\newcommand{\bs}{\mathbf{s}}
\newcommand{\bSigma}{\bm{\Sigma}}

\newcommand{\btau}{\bm{\tau}}
\newcommand{\btheta}{\bm{\theta}}
\newcommand{\bTheta}{\bm{\Theta}}

\newcommand{\bu}{\mathbf{u}}
\newcommand{\bU}{\mathbf{U}}
\newcommand{\bW}{\mathbf{W}}
\newcommand{\bw}{\mathbf{w}}
\newcommand{\bx}{\mathbf{x}}
\newcommand{\bxi}{\bm{\xi}}
\newcommand{\bX}{{\mathbf{X}}}
\newcommand{\by}{\mathbf{y}}
\newcommand{\bY}{\mathbf{Y}}
\newcommand{\bz}{\mathbf{z}}
\newcommand{\bZ}{\mathbf{Z}}

\newcommand{\Prob}{\text{P}}
\newcommand{\R}{\mathbb{R}}

\newcommand\myeq{\mathrel{\stackrel{\makebox[0pt]{\mbox{\normalfont\tiny d}}}{=}}}

\theoremstyle{plain}

\newtheorem{prop}{Proposition}

\theoremstyle{definition}

\newtheorem{defi}{Definition}

\begin{document}

\thispagestyle{empty} \baselineskip=28pt \vskip 5mm
\begin{center} {\Huge{\bf A Generalized Unified Skew-Normal Process with Neural Bayes Inference}}
	
\end{center}

\baselineskip=12pt \vskip 10
mm

\begin{center}\large
Kesen Wang\footnote[1]{\baselineskip=10pt Statistics Program,
King Abdullah University of Science and Technology,
Thuwal 23955-6900, Saudi Arabia.\\
E-mail: kesen.wang@kaust.edu.sa, marc.genton@kaust.edu.sa
} and Marc G. Genton\textcolor{blue}{$^{1}$}
\end{center}
\baselineskip=17pt \vskip 4mm \centerline{\today} \vskip 7mm

\begin{center}
{\large{\bf Abstract}}

\end{center}

In recent decades, statisticians have been increasingly { encountering} spatial data that exhibit non-Gaussian behaviors such as asymmetry and { heavy-tailedness}.  As a result, the assumptions of symmetry and fixed tail weight in Gaussian processes have become restrictive and may fail to capture the intrinsic properties of the data. To address the limitations of the Gaussian models, a variety of skewed models has been proposed, of which the popularity has grown rapidly. These skewed models introduce parameters that govern skewness and tail weight. Among various proposals in the literature, unified skewed distributions, such as the Unified Skew-Normal (SUN), have received considerable attention. In this work, we revisit a more concise and intepretable re-parameterization of the SUN distribution and apply the distribution to random fields by constructing a generalized unified skew-normal (GSUN) spatial process. We demonstrate { that the GSUN is a valid spatial process by showing its vanishing correlation in large distances} and provide the corresponding spatial interpolation method.  In addition, we develop an inference mechanism for the GSUN process using the concept of neural Bayes estimators with deep graphical attention networks (GATs) and encoder transformer. We show the superiority of our proposed estimator over the conventional CNN-based architectures regarding stability and accuracy by means of a simulation study and application to Pb-contaminated soil data. Furthermore, we show that the GSUN process is different from the conventional Gaussian processes and  Tukey $g$-and-$h$ processes, through the probability integral transform (PIT).

\baselineskip=14pt

\par\vfill\noindent

\clearpage\pagebreak\newpage \pagenumbering{arabic}
\baselineskip=26pt
\section{Introduction}

In the past decade, there has been a rising interest in applying the Skew-Normal (SN) distribution to model a spatial field because of the skewness parameter that can { capture} asymmetric data. In particular, \cite{kim2004bayesian} constructed a Bayesian predictor based on the SN distribution. In addition, \cite{dominguez2003multivariate}, \cite{allard2007new}, and \cite{rimstad2014skew} proposed to use the Closed Skew-Normal (CSN), a generalization of the SN distribution, to model skewed geostatistical data with various methods such as MCMC and method of moments presented for parameter estimation. Despite these fruitful studies, \cite{arellano2006unification} mentioned the non-identifiability and over-parameterization issue within the CSN distribution. Moreover, \cite{genton2012identifiability} conducted a comprehensive study on the non-identifiability problem of the SN spatial process and extended the conclusion to elliptically contoured random fields. Furthermore, \cite{minozzo2012existence} demonstrated that SN-type spatial processes do not satisfy the requirement of { vanishing correlations in large distances.} A possible remedy was developed by \cite{zhang2010spatial}, where a stationary process was constructed with SN as marginal distributions.  

To allow the spatial model to capture asymmetry and heavy tails, \cite{tagle2019non} applied the multivariate Skew-$t$ (ST) distribution to model daily wind data. Moreover, \cite{bevilacqua2021non} developed a ST process model that addresses the regression and dependence analysis of asymmetric and heavy-tailed spatial data. Although the ST spatial process possesses numerous desirable properties, it suffers from the same problematic { non-vanishing correlation in large distances} as the SN process. 

Apart from SN or ST spatial models, practitioners also sought to construct spatial models using the Unified Skew-Normal (SUN) distribution. To this end, a comprehensive summary of the existence of the SN and SUN family-based random fields can be found in \cite{mahmoudian2018existence}. A particularly noteworthy SUN-based spatial process is the Unified Skew Gaussian-Log-Gaussian (SUGLG) process proposed in \cite{zareifard2013non} with a Stochastic Approximation Expectation Maximization (SAEM) algorithm for parameter inference.  The SUGLG { process} is a more complicated model than the Gaussian process and can be used to model the asymmetric property of the spatial data and the SAEM algorithm was proposed to avoid the intractable computation of the integrals involved in the log-likelihood function. In addition, correlations vanish in large distances in the SUGLG process  because it is a convolution of a Gaussian process { (dominant) and a form of truncated Gaussian process (latent), where both processes follow such property.} Although the SUGLG can address numerous challenges existing in the SN and ST processes, the SAEM is likely to be unstable and time-consuming with large spatial dimensions. The simulation study and real data application conducted in \cite{zareifard2013non} were limited to sample sizes of about 100.  Moreover, the SUGLG model is restricted to a simplified sub-model of the SUN discussed in \cite{nonke}, which imposes the same dependence structure on the covariance of the observed and latent processes and the skewness matrix. Hence, in this work, we generalize the SUGLG and propose the Generalized Unified Skew-Normal (GSUN) process by involving a { separate} dependence structure also in the latent dimensions and designing a principled diagonal structure on the skewness matrix to improve its statistical interpretability and maintain vanishing correlations in large distances of the process. The skewness parameter accounts for information from both the observed and latent processes instead of the observed alone. Furthermore, we adopt a re-parameterization of the SUN distribution; the new parameterization is numerically more stable and geometrically more intuitive compared with the original parameterization proposed in \cite{arellano2006unification}. 

 \cite{sainsbury2023neural} demonstrated that a Graph Neural Network (GNN) can accurately approximate Bayes estimators, a function that maps the realizations of spatial processes to the point estimates of the corresponding parameters, which is amortized and efficient in inference. The idea of using neural networks to map realizations to the point estimates of the parameters is defined as the neural Bayes estimator \citep{sainsbury2024likelihood}. Furthermore, \cite{lenzi2023neural}, \cite{richards2023likelihood}, and \cite{walchessen2024neural} have applied { variants of the neural Bayes estimators} to the inference of complex processes and distributions, for which the likelihood functions are intractable or computationally intensive. A comprehensive review of neural Bayes estimators and amortized parameter inference can be found in \cite{zammit2024neural}. Following this direction, we combine a deep Graph Attention Network (GAT) \citep{velickovic2018graph} with a multi-layer encoder transformer architecture \citep{vaswani2017attention} to construct the neural Bayes estimator for the GSUN model, which is adapted to regularly and irregularly spaced grids. Moreover, we developed a simulation-based uncertainty quantification method for the parameter estimates.

This paper is organized as follows. Section~\ref{SUN_recap} revisits the SUN distribution and the more concise re-parameterization, laying down the ground for the proposition of the GSUN model. Section~\ref{GSUN_model} formally introduces the GSUN spatial process including its definition, validity, and spatial interpolation. Section~\ref{neural_4} gives a comprehensive background introduction to the neural Bayes estimator. Section~\ref{arc_rep} presents the network architecture and data representation methods adopted in this work. Section~\ref{simulation} conducts a series of comparative and evaluative simulation studies of the GSUN process to demonstrate its efficiency. Section~\ref{real_data} applies the GSUN model to a dataset of Pb-contaminated soils and draws a comparison between the GSUN and SUGLG model. Section~\ref{conclude} concludes and outlines the contributions of this work.
\section{Revisiting the SUN Distribution} \label{SUN_recap}
In this section, we briefly revisit the Unified Skew-Normal (SUN) distribution proposed in \cite{arellano2006unification} and introduce a more efficient parameterization documented in \cite{AG22}.  

\subsection{The SUN Distribution} \label{SUN_recap_1}
As previously mentioned, to address the non-identifiability issue existing in the CSN distribution, \cite{arellano2006unification} proposed a unified framework involving many skewed models as sub-cases by enforcing $\bar\bOmega^*$ below to be a correlation matrix. The SUN distribution arises from the selection mechanism described in \cite{arellano2006unification}, starting from a joint multivariate Gaussian random vector
\begin{align*}
    \bU = \begin{pmatrix}\bU_0\\
              \bU_1\end{pmatrix} 
    \sim {\cal N}_{m+d}(\0,\bar\bOmega^*), \qquad 
    \bar\bOmega^* = \begin{pmatrix}
                \Bar{\bGamma} & \bDelta^\top \\
                \bDelta & \Bar{\bOmega}
    \end{pmatrix}. 
\end{align*}
 Then, a random vector $\bZ = (\bU_1|\bU_0 + \btau > \0)$, where $\btau\in\R^m$ is a truncation parameter and $\bDelta$ represents the skewness matrix, is a non-shifted SUN random vector. With a local shift parameter $\bxi$,  $\bY=\bxi+\bomega\,\bZ$, where $\bOmega=\bomega\Bar{\bOmega}\bomega$, is a fully parameterized SUN random vector.  The probability density function (pdf) of $\bY$ is
\begin{align}
    f(\by)=\phi_d(\by-\bxi;\bOmega)\:\frac{\Phi_m(\btau + \bDelta^\top\Bar{\bOmega}^{-1}\bomega^{-1}(\by-\bxi);\Bar{\bGamma}-\bDelta^\top\Bar{\bOmega}^{-1}\bDelta)}{\Phi_m(\btau;\Bar{\bGamma})},
    \qquad \by\in\R^d.
    \label{sun-pdf}
\end{align}
In this work, we denote the SUN random vector with pdf defined in \eqref{sun-pdf} as ${\cal SUN}_{d,m}(\bxi,\bOmega,\bDelta,\btau,\bar\bGamma)$.
Numerous interesting and practical properties such as moments of the SUN are studied by \cite{arellano2006unification}, \cite{gupta2013some}, \cite{AA22}, and \cite{nonke}. 

\subsection{Re-parameterization of the SUN Distribution} \label{re-para}
In Section~\ref{SUN_recap_1}, we have denoted a SUN random vector by $\bY \sim {\cal SUN}_{d,m}(\bxi,\bOmega,\bDelta,\btau,\Bar{\bGamma})$, which is the original parameterization proposed in \cite{arellano2006unification}. Although such a parameterization bears many intriguing advantages,  the covariance matrix $\bOmega^*$ induces numerical instability in simulating realizations. Such a numerical instability can be observed from the \textit{rsun} function provided in the \textit{sn} $R$-package \citep{sn}, which simulates replicates of the SUN random vector using exactly this particular parameterization and stochastic representation. The  simulations often run into issues with rather large values of the entries in $\bDelta$, causing $\bOmega^*$ to be singular and hence to fail in Cholesky factorization, a crucial step to produce the samples. Alternatively, $\bY$ can also be represented with the convolution mechanism proposed in \citet{arellano2006unification}. In this sense, introducing two intermediate matrices $\textbf{B}_0$ and $\textbf{B}_1$ obscures the directions of the skewness of the random vector. In other words, $\bDelta$ introduces the skewness but fails to indicate the direction, making $\bDelta$ less interpretable. Meanwhile, this parameterization complicates simulations of random samples from the convolution mechanism.

To address these disadvantages, \cite{AG22} provided a simple way to construct a random vector with SUN distribution. In detail, they let $\boldsymbol{\omega}\boldsymbol{\Delta}=\bH\Bar{\bGamma}$ and 
$\bOmega=\bPsi+\bH\Bar{\bGamma}\bH^\top$. Then, they re-parameterized the definition of unified skew-normal distribution in \cite{AA22} using the convolution mechanism proposed in \cite{arellano2006unification} and denoted it as ${\cal SUN}_{d,m}(\bxi,\bPsi,\bH,\btau,\Bar{\bGamma})$. Here is a formal definition: 
\begin{defi}\label{defSUN}
A $d$-dimensional random vector $\bY$ has a unified skew-normal distribution, denoted as $\bY \sim {\cal SUN}_{d,m}(\bxi,\bPsi,\bH,\btau,\Bar{\bGamma})$, with location parameter \(\bxi\in \R^p\), dispersion matrix \(\bPsi \in \R^{d\times d}\), shape/skewness parameters \(\bH \in \R^{d \times m}\), latent truncation parameter \(\btau \in \R^m\), and positive definite latent correlation matrix \(\Bar{\bGamma} \in \R^{m \times m}\), if \(\bY=\bxi+\bH\bU+\bW\), where \(\bU\myeq(\bW_0|\bW_0 + \btau > \0)\), with \(\bW_0 \sim {\cal N}_m(\0,\Bar{\bGamma})\) and \(\bW \sim {\cal N}_d(\0,\bPsi)\) having independent, multivariate, Gaussian distributions.
\end{defi}

The new parameterization leads to the same distribution as with the original, see Section~S.1 in the Supplementary Materials for a detailed proof. Moreover, the new parameterization is numerically stable and eliminates the issues of Cholesky factorization in the original parameterization. In addition, the new skewness parameter $\bH$ provides a direct representation of the direction of the skewed mass. The convolution representation $\bY = \bxi + \bH\bU + \bW$ can be understood as a symmetric normal random vector $\bW$ plus a truncated normal random vector $\bU$, with $\bH$ determining the asymmetry introduced by $\bU$ and controlling the direction.

Lastly, the re-parameterization also allows for a more straightforward computation of the mean and variance. Because $\bY = \bxi + \bH\bU + \bW$, then $\mathbb{E}(\bY) = \bxi + \bH\mathbb{E}(\bU)$ and $\text{var}(\bY) = \bH\text{var}(\bU)\bH^\top + \bPsi$. Note that the mean and variance of multivariate truncated normal random vectors can be computed numerically with the $R$-package \textit{tmvtnorm} \citep{tmvtnorm}. 

\section{The GSUN Spatial Model} \label{GSUN_model}
In this section, we study the GSUN spatial model by first defining the GSUN spatial process and, then, deriving the kriging formula for spatial interpolation.
\subsection{GSUN Spatial Processes}
\cite{zhang2010spatial} discussed how the SN spatial model faces a non-ergodic issue, which is crucial for a valid and reasonable spatial process. The non-ergodic issue can also be generalized to the ST spatial model. The problem that causes SN and ST to be non-ergodic is that all locations in the study region share an identical univariate truncated normal or truncated $t$ random variable. As a result, the covariance or correlation between two locations persists even if they are significantly distant from each other.

In this section, we propose a more flexible construction of the SUN spatial model that satisfies the requirement of vanishing correlations in large distances and accounts for the latent spatial process of skewness. In detail, we define a GSUN spatial process as:
\begin{align*}
    Z(\bs) = W(\bs) + h(\bs)W^{+}(\bs),
\end{align*}
where $W(\bs)$ and $W^+(\bs) \overset{d}{=} (U(\bs)|U(\bs) > 0)$, with $U(\bs)$ being a Gaussian process, are independent Gaussian and non-negative Gaussian spatial processes with Mat\'ern covariance and correlation functions (see \cite{wang2023parameterization} for the choice of parameterization), respectively. Here $h(\bs)$, a location-specific scalar that dictates the amount of the non-negative Gaussian process $W^+(\bs)$ added to $W(\bs)$, { can be specified with any legitimate real values}. If observed at $n$ locations, then let $\bZ = \{Z(\bs_1),\dots,Z(\bs_n)\}^\top$, $\bW = \{W(\bs_1),\dots,W(\bs_n)\}^\top$, and $\bW^+ = \{W^+(\bs_1),\dots,W^+(\bs_n)\}^\top$. We have $\bW \sim {\cal N}_n(\0,\bSigma(\btheta_1))$ and $\bW^+ \sim {\cal TN}_n(\0;\0, \textbf{C}(\btheta_2))$, where ${\cal TN}_n(\textbf{a};\textbf{m}, \textbf{K})$ represents an $n$-dimensional truncated Gaussian distribution cut from below $\textbf{a}$ with mean vector $\textbf{m}$ and covariance matrix $\textbf{K}$; $\bSigma(\btheta_1)$ and $\textbf{C}(\btheta_2)$ are the Mat\'ern covariance and correlation matrices  with $\btheta_1 = (\sigma^2,\beta_1,\nu_1)^\top$ and $\btheta_2=(\beta_2,\nu_2)^\top$. In addition, we set $\bH = \text{diag}(h(\bs_1),\dots,h(\bs_n)) = \delta_1 \bI_n + \delta_2 \text{diag}\{(\sum_{i=1}^{n} \lambda_{i,\bSigma(\btheta_1)}\bP_{i,\bSigma(\btheta_1)} + \lambda_{i,\textbf{C}(\btheta_2)}\bP_{i,\textbf{C}(\btheta_2)})/(2n)\}$, where 
$(\lambda_{i,\bSigma(\btheta_1)},\bP_{i,\bSigma(\btheta_1)})$ and $(\lambda_{i,\textbf{C}(\btheta_2)},\bP_{i,\bSigma(\btheta_1)})$  are the $i$-th eigenvalue and normalized eigenvector of the covariance matrix $\bSigma(\btheta_1)$ and correlation matrix $\textbf{C}(\btheta_2)$. Here, $\delta_1 \in \R$ and $\delta_2 \in \R$ control the common skewness and the skewness induced from the principal components. { It is worth noticing that under this particular setting of $h(\bs)$, when $\delta_2 = 0$, the GSUN process is stationary because $h(\bs)$ will be the same across all locations. Otherwise, the GSUN process is non-stationary since the weighted principle components will mostly likely have varying marginal components and hence result in differing $h(\bs)$ over different locations.} In symbols, we have $\bZ \sim {\cal SUN}_{n,n}(\0,\bSigma(\btheta_1),\bH,\0,\textbf{C}(\btheta_2))$, hence $\bZ = \bW + \bH\bW^+$. Our model, compared to the SUN spatial model in \cite{zareifard2013non}, allows more flexibility for the latent skewed process $W^+(\bs)$. This process has its own distinct range and smoothness parameters, extending the dependence modeling capacity of the GSUN spatial process to the latent effects. For example, if there is extremely heavy rainfall in one region, it is likely to be similar in neighboring areas. Moreover, defining $\bH$ as a diagonal matrix composed of a linear combination of a common skewness effect and a scaled vector of a weighted average of the eigenvectors makes more statistical sense than simply scaling $\bSigma(\btheta_1)$ by a constant as introduced in \cite{zareifard2013non}. In such construction, $\bH$ can account for skewness in a shared direction and the direction of the principal components weighted by their eigenvalues. 

Now we demonstrate the { vanishing correlations in large distances} of the proposed models. First,
\begin{align*}
\text{cov}\{Z(\bs_i), Z(\bs_j)\} =  {\bSigma(\btheta_1})_{i,j} + \bH_{i,i} \text{var}(\bU)_{i,j}\bH_{j,j}.
\end{align*}
If we assume that locations $\bs_i$ and $\bs_j$ are sufficiently distant from each other, then ${\bSigma(\btheta_1})_{i,j}$ and $\text{var}(\bU)_{i,j}$ are { converging to} 0 due to the Mat\'ern covariance function.  Hence, $\text{cov}\{Z(\bs_i), Z(\bs_j)\} \approx~0$. 

The reason for $\text{var}(\bU)_{i,j}$ to be { close to} 0 at significantly large distances is that $\text{var}(\bU)$ is a covariance matrix of a multivariate truncated Gaussian random vector, where the multivariate Gaussian random vector has covariance matrix $\textbf{C}(\btheta_2)$. For Gaussian random vectors, if the marginals are uncorrelated, they are also independent and the truncation operation bears no effect on the induced independence (see Proposition \ref{trunc_ind} in the Supplementary Materials). 

\subsection{Spatial Interpolation}
We can rely on the kriging method for spatial interpolation. For instance, we note that the SUN has closed conditional distributions \citep{AA22, wang2024multivariate}. Let $$
\begin{pmatrix}
    \bZ_1 \\
    \bZ_2
\end{pmatrix} \sim {\cal SUN}_{n_1 + n_2,n}\left(\begin{pmatrix}
    \bxi_1 \\
    \bxi_2
\end{pmatrix}, \begin{pmatrix}
    \bPsi_{11} & \bPsi_{12} \\
    \bPsi_{21} & \bPsi_{22}
\end{pmatrix},\begin{pmatrix}
    \bH_1 \\
    \bH_2
\end{pmatrix},\btau,\bar\bGamma\right)
$$ with $n=n_1 + n_2$. Then, following from Proposition 3.2 in \cite{arellano2010multivariate} with a Gaussian density generator, we can plug in the new parameterization proposed in Section~\ref{re-para} to obtain the conditional distribution of the SUN as $(\bZ_1|\bZ_2 = \bz_2) \sim {\cal SUN}_{n_1,n}(\bxi_{1 \cdot 2},\bPsi_{1 \cdot 2}, \bH_{1 \cdot 2}, \btau_{1 \cdot 2}, \bar\bGamma_{1 \cdot 2})$, where 
\begin{align*}
    \bxi_{1 \cdot 2} & = \bxi_1 + \left(\bPsi_{12}\bPsi_{22}^{-1}+\bH_{1.2}\bar\bGamma_{1.2}\bH_{1.2}^\top\bPsi_{22}^{-1}\right)(\bz_2 - \bxi_2), \\
    \bH_{1 \cdot 2} & = (\bH_1 - \bPsi_{12}\bPsi_{22}^{-1}\bH_2) \bgamma_{1 \cdot 2}, \\
    \bPsi_{1 \cdot 2} & = \bPsi_{11} - \bPsi_{12}\bPsi_{22}^{-1}\bPsi_{21},\\
    \btau_{1 \cdot 2} & = \bgamma_{1 \cdot 2}^{-1}\{\btau + \bGamma_{1 \cdot 2}\bH_2^\top\bPsi_{22}^{-1}(\bz_2 - \bxi_2)\}, \\
    \bar\bGamma_{1 \cdot 2} & = \bgamma_{1 \cdot 2}^{-1}(\bar\bGamma^{-1} + \bH_2^\top \bPsi_{22}^{-1}\bH_2)^{-1}\bgamma_{1 \cdot 2}^{-1},\\
    \bgamma_{1 \cdot 2} & = \text{diag}\{(\bar\bGamma^{-1} + \bH_2^\top \bPsi_{22}^{-1}\bH_2)^{-1}\}^{1/2}. 
\end{align*}

Here, $\mathbf{Z}_1$ represents the values to be predicted, and $\mathbf{z}_2$ represents the observed values. { Interpolated} values can be obtained by calculating the first moment of $\mathbf{Z}_1|\mathbf{Z}_2$. Additionally, we can determine location-specific prediction intervals by computing the covariance matrix of $\mathbf{Z}_1|\mathbf{Z}_2$ and extracting the marginal variances from the diagonal. Both computations are efficient and straightforward as described in Section~\ref{re-para}. In addition, the computations of conditional distribution and expectation are implemented in the $sn$ $R$-package \citep{sn}.

\section{Neural Bayes Estimators} \label{neural_4}
\subsection{Background} \label{bayes_back}
Parameter inference for high-dimensional spatial models has long been challenging due to the significant computational costs involved in evaluating high-dimensional multivariate probability and cumulative distribution functions. The GSUN model, like others, requires the assessment of both multivariate Gaussian density and the ratios of two multivariate cumulative Gaussian distributions. Deep learning methods have become more popular as a potential solution to mitigate the heavy computation costs associated with distribution functions. In brief, these approaches mainly focus on approximating the likelihood function, the likelihood-to-evidence ratio, and the posterior distribution using deep neural networks \citep{sainsbury2023neural}. Furthermore, \cite{zammit2024neural} provided a detailed and comprehensive discussion of the theoretical basis and applications of deep learning methodologies in statistical inference and modeling.

The recent advancements in deep learning have led to the development of a new parameter { point} estimator called the neural Bayes estimator. This estimator uses deep neural networks to map a set of data samples to estimates of the underlying parameters. As described in the study by \cite{sainsbury2023neural}, the neural Bayes estimator is computationally efficient, approximately Bayes, likelihood-free, and amortized. Once trained on sufficient data, it allows for much faster inference compared to traditional methods like Markov Chain Monte Carlo (MCMC) or maximum likelihood estimation. The neural Bayes estimator was initially used in spatial statistics with deep Convolutional Neural Networks (CNNs) introduced in Chapter~9 of \cite{goodfellow2016deep}. However, this approach was limited to spatial data on regularly spaced grids. \cite{sainsbury2023neural} extended the capabilities of the neural Bayes estimator to work with irregularly spaced grids using graphical neural networks, greatly expanding its potential applications.

\subsection{Link to Bayes Risk}
 In this section, we aim to expand the concept of the neural Bayes estimator to our GSUN process. In particular, we have that 
\begin{align} \label{min_neural_bayes}
    \hat\bbeta \equiv \underset{\bbeta}{\text{arg min}} \quad \frac{1}{k}\sum_{k=1}^K  L(\bTheta^k,\hat\bTheta_{(\bz_k,\bbeta)}),
\end{align}
where $\bTheta = (\btheta_1,\btheta_2, \delta_1, \delta_2)$. Here, $\hat\bTheta_{(\bz_k,\bbeta)}$ denotes the point estimate of $\bTheta$ that depends on $\bz_k$ and $\bbeta$, where $\bz_k$ is the $k$-th realization of the GSUN spatial process with parameter $\bTheta^k$ and $\bbeta$ represents the weights of the deep neural network. Here, $L(\cdot,\cdot)$ is a non-negative loss function. As mentioned in \cite{sainsbury2023neural}, \eqref{min_neural_bayes} is a Monte Carlo approximation of the Bayes risk of a proposed estimator:
\begin{align} \label{bayes_risk}
    \int_{\bTheta}\int_{\bZ} L(\bTheta,\hat\bTheta_{(\bZ,\bbeta)})f(\bZ|\bTheta)\text{d}\bz \text{d}\Pi(\bTheta),
\end{align}
where $\Pi(\cdot)$ is a prior measure. It is quite clear that equation \eqref{bayes_risk} cannot be minimized through analytical methods. As a result, the Bayes estimator typically does not have a closed-form solution, which is why the concept of deep neural networks and { empirical risk minimization} are utilized. According to the universal approximation theorem \citep{hornik1989multilayer, zhou2020universality}, a sufficiently large and deep neural network can approximate complex functions, making it suitable for approximating the Bayes estimator. In practical terms, minimizing \eqref{min_neural_bayes} is effectively the same as minimizing \eqref{bayes_risk} when a sufficiently large training size { $K$} is available.
\subsection{Algorithm}
\begin{algorithm}[b!]
\caption{Training and Inference of Neural Bayes Estimators}\label{alg:bayes}
\textbf{Training} \\
\textbf{Require}: Sample size $n$, number of replicates $N$, spatial model $f(\bZ|\bTheta)$, prior $\Pi(\bTheta)$, neural network architecture for $\hat\bTheta_{(\bz^n,\bbeta)}$, $L(\cdot,\cdot)$ a non-negative loss function, a dropout rate $D$\\
\textbf{Procedure}:\\
\hspace*{5mm} (1) Simulate $\bTheta \sim \Pi(\bTheta)$ \\
\hspace*{5mm} (2) Simulate $N$ replicates of $\bZ^n \sim f(\bZ|\bTheta)$ \\
\hspace*{5mm} (3) Compute $\hat\bTheta_{(\bz^n,\bbeta)}$ with dropout rate $D$ on each layer \\
\hspace*{5mm} (4) Compute $L(\bTheta,\hat\bTheta_{(\bz^n,\bbeta)})$ \\
\hspace*{5mm} (5) Backward propagate and update $\bbeta$ \\
\hspace*{5mm} (6) Repeat (1) - (5) until a pre-specified criterion is met
\\
\textbf{Stopping Criterion}: \\
\hspace*{5mm}  The loss 
$L(\bTheta,\hat\bTheta_{(\bz^n,\bbeta)})<\epsilon_L$ { for $L$ consecutive simulations}\\
\\
\textbf{Inference}\\
\textbf{Require}: the observed sample $\bz$\\
\textbf{Procedure}:\\
\hspace*{5mm} Plug in the sample $\bz$ into the neural networks directly and obtain $\hat\bTheta_{(\bz^n,\bbeta)}$
\end{algorithm}
The training process of a neural Bayes estimator needs to involve a large number of simulations, { $K$}. Specifically, we start by simulating $\bTheta^k$ from its prior and then simulate $\bZ^{k,n}$, where $n$ denotes the sample size { and $k$ represents the count of training samples}, from the corresponding distribution. When it comes to choosing priors, \cite{lenzi2023neural} advised against using priors with concentrated weights on a particular value, such as the Gaussian prior. Therefore, in our training process, we use uniform priors for all parameters. Additionally, given the need to repeatedly simulate new parameters and data for the neural network, an \say{on-the-fly} simulation is an advantageous approach to prevent overfitting the network \citep{sainsbury2024likelihood} on specific data samples. In other words, for each training step, we simulate $\bTheta^k$ and $\bZ^{k,n}$ from scratch and then input the data into the proposed network to update $\bbeta$.

One common challenge in spatial statistics is the limited number of replicates $N$, which has led to the development of various covariance functions (for more information, refer to \cite{wang2023parameterization}). Conventionally, we can assume that $N$ = 1; however, based on the empirical evidence from the training stage, it appears that setting $N$ to a slightly larger value, such as 10, can accelerate the learning process. This could be because averaging across repetitions stabilizes the underlying characteristics of the generated data. As a result, we have implemented this approach in the training process of our neural Bayes estimator. Additionally, based on the findings of \cite{srivastava2014dropout}, randomly deactivating certain neurons during training can help prevent overfitting. Therefore, we have also introduced a dropout rate $D$ as a hyper-parameter; refer to Algorithm~\ref{alg:bayes} for a detailed description of the training algorithm.

As described in Algorithm \ref{alg:bayes}, the training stage involves numerous sampling repetitions of $\bTheta^k$ and subsequently $\bZ$ to achieve a uniform and sufficient coverage of the ranges of all parameters. Consequently, the Neural Bayes Estimator may need millions of training samples and a considerable amount of time to be properly trained. However, once the estimator is trained, the inference process becomes quite efficient because it is a straightforward plug-in procedure, which is also called  \textit{amortized inference} \citep{sainsbury2023neural}.

Additionally, when simulating $\bTheta$ from its prior distribution, it is important to avoid $(\delta_1,\delta_2)$ being close to $0$ at the same time. Indeed, if $(\delta_1,\delta_2)$ are both $0$, then the GSUN process reduces to a conventional Gaussian process and, therefore, the parameters $(\beta_2,\nu_2)$ become void and non-estimable.

\section{Architecture and Spatial Data Representation} \label{arc_rep}
In this section, we presents the two major blocks of our neural Bayes estimator for the GSUN process.

\subsection{Graphical Attention Networks}

Graphical Neural Networks (GNNs) arise naturally as a solution in the case of spatial processes. The realization of the study area can be represented directly as an undirected graph, where the nodes are specific locations, and the values are the corresponding realizations. The edges are drawn between two nodes if they are neighbors, thus representing the spatial information. Recent advancement in graphical neural networks \citep{kipf2016semi, gilmer2017neural} provided a framework to extend deep learning towards graph-structured data whether the goal is a node-level, edge-level or graph-level task. Among all graphical neural networks, the Graphical Attention Networks (GATs) by \cite{velickovic2018graph} significantly outperformed the previous GNN architectures such as the Graphical Convolutional Networks (GCNs) by \cite{kipf2016semi}. The main novelty of GATs is the introduction of an attention mechanism that allows the model to adaptively weight node neighbors concerning their importance; hence, more flexible and expressive representations are possible.

Consider an undirected graph $\mathcal{G} = (\mathcal{V}, \mathcal{E})$, where $\mathcal{V}$ is the set of nodes and $\mathcal{E}$ is the set of edges. Let $N_i = \{ j \in \mathcal{V} | (i,j) \in \mathcal{E} \}$ denote the set of neighbors of node $i$. Each node $i \in \mathcal{V}$ is associated with a feature vector $\mathbf{F}_i \in \mathbb{R}^f$, where $f$ is the dimensionality of the input feature space. The goal of GCNs is to learn the node embeddings $\mathbf{v}_i \in \mathbb{R}^{f'}$, where $f'$ is the dimensionality of the output feature space - by iteratively aggregating feature information from neighboring nodes. A common GCN layer aggregates node information by performing a weighted sum of neighboring node features:
\begin{equation*}
\mathbf{F}_i^{(l+1)} = \sigma \left( \sum_{j \in N_i \cup \{i\}} \frac{1}{\sqrt{d_i d_j}} \mathbf{G}^{(l)} \mathbf{F}_j^{(l)} \right),
\end{equation*}
where \( \mathbf{G}^{l} \) is a trainable weight matrix at layer \( l \), \( d_i \) and \( d_j \) denote the degrees of node \( i \) and \( j \), and \( \sigma \) is a nonlinear activation function. Although this method takes advantage of information from a local neighbourhood, it does so in a very uniform way, without discriminating between different nodes' neighbours, and this could make it difficult for the model to learn complex dependencies among nodes.

GATs remove this limitation by introducing the self-attention mechanism, allowing the nodes to assign different weights to their neighbors. For any pair of nodes $i$ and $j$, we first compute the unnormalized attention coefficient:
\begin{equation*}
e_{ij} = \text{LeakyReLU} \left\{ \mathbf{a}^\top (\mathbf{G} \mathbf{F}_i \| \mathbf{G} \mathbf{F}_j) \right\},
\end{equation*}
where $\mathbf{G} \in \mathbb{R}^{f'\times f}$ is a learnable weight matrix shared for all nodes, $\mathbf{a} \in \mathbb{R}^{2f'}$ is the attention vector, and $\|$ denotes the vectorization of the feature vectors of nodes $i$ and $j$. Herein, $\text{LeakyReLU}(x) = \max(0.2x, x)$ introduces the non-linearity. The attention coefficient $e_{ij}$ reflects the importance of node $j$'s features to node $i$.
The corresponding attention coefficients are then normalized across all neighbors of node $i$ using the softmax function:
\begin{equation} \label{softmax}
\alpha_{ij} = \frac{\exp(e_{ij})}{\sum_{k \in N_i} \exp(e_{ik})}.
\end{equation}
By considering these normalized attention coefficients, the updated representation for node $i$ is computed as a weighted sum of feature vectors of its neighbors:
\begin{equation*}
\mathbf{F}_i^{(l+1)} = \sigma \left( \sum_{j \in N_i \cup \{i\}} \alpha_{ij} \mathbf{G} \mathbf{F}_j^{(l)} \right),
\end{equation*}
where $\sigma(\cdot)$ denotes a non-linear activation function. The attention mechanism allows the model to learn which neighbors are more relevant for the task at hand, and provides more flexibility compared to GCNs.

In order to enhance the learning capacity of the model, GAT incorporates multi-head attention: it deploys $K$ independent attention mechanisms to calculate multiple series of attention coefficients; all resulting node representations obtained from each head are then combined by concatenation (intermediate layers) or averaging (when the layer is the last). Formally, multi-head mechanisms were defined as: 
\begin{equation*} 
\mathbf{F}_i^{(l+1)} = \|_{k=1}^K \sigma \left( \sum_{j \in N_i \cup {i}} \alpha_{ij}^{(k)} \mathbf{G}^{(k)} \mathbf{F}_j^{(l)} \right), 
\end{equation*} 
where $\|_{k=1}^K$ denotes the vectorization operation and $\alpha_{ij}^{(k)}$ are the attention coefficients from the $k$-th attention head. This approach improves the model's expressiveness because the model can grasp different aspects of node relationships through multiple attention heads.

\subsection{Encoder Transformer} \label{encoder}

The state-of-the-art Encoder architecture \citep{vaswani2017attention} has founded many approaches for modeling dependence structures, of which the main novelty is in using self-attention mechanisms to efficiently model long-range dependencies without any recurrence nor convolution.

The architecture of the Encoder consists of a stack of $L$ identical layers, each having two major sub-components: a multi-head self-attention mechanism and a feed-forward neural network (FFN).

The self-attention mechanism enables the model to learn the latent dependence structure of the data by first projecting the input data into three distinct spaces: the query ($\mathbf{Q}$), key ($\mathbf{K}$), and value ($\mathbf{V}$), where
\begin{equation*}
\mathbf{Q} = \mathbf{XW}_Q, \quad \mathbf{K} = \mathbf{XW}_K, \quad \mathbf{V} = \mathbf{XW}_V,
\end{equation*}
with $\mathbf{W}_Q, \mathbf{W}_K, \mathbf{W}_V \in \mathbb{R}^{d \times d_k}$ as learnable weight matrices. Then, the mechanism computes the scaled outer product between the queries and keys, which is defined as the attention scores:
\begin{equation*}
\text{Attention}(\mathbf{Q}, \mathbf{K}, \mathbf{V}) = \text{softmax} \left( \frac{\mathbf{QK}^\top}{\sqrt{d_k}} \right) \mathbf{V}.
\end{equation*}
Here, softmax($\cdot$) is identical to the softmax function mentioned in equation \eqref{softmax} but is applied row-wise to $\textbf{Q}\textbf{K}^\top/\sqrt{d_k}$. The attention scores are divided by $\sqrt{d_k}$ to ensure identical scales of the outer products and, therefore, even distribution of the attention scores when a softmax function is applied. In addition, such operation also stabilizes the gradients during training.

The Encoder utilizes multi-attention mechanism as a method for enlarging the model's capacity for capturing relations of various types. Instead of one Attention block, it has $h$ parallel attention heads, each performing attention with an independent query, key, and value matrices. The outputs of attention serve as vectors which are further concatenated and linearly transformed:
\begin{equation*}
\textbf{h} = \text{MultiHead}(\mathbf{Q}, \mathbf{K}, \mathbf{V}) = [\text{head}_1, \dots, \text{head}_h] \mathbf{W}_O,
\end{equation*}
where each attention head is computed as:
\begin{equation*}
\text{head}_i = \text{Attention}(\mathbf{Q}_i, \mathbf{K}_i, \mathbf{V}_i),
\end{equation*}
and $\mathbf{W}_O \in \mathbb{R}^{hd_k \times d}$ is a learnable output projection matrix.

Following the multi-head self-attention, each layer of the Encoder applies a FFN to the learned embeddings. In particular, this FFN consists of two linear transformations with a ReLU activation in between:
\begin{equation*}
\text{FFN}(\mathbf{h}) = \text{ReLU}(\mathbf{h} \mathbf{W}_1 + \mathbf{b}_1) \mathbf{W}_2 + \mathbf{b}_2,
\end{equation*}
where $\mathbf{W}_1 \in \mathbb{R}^{d \times d_{ff}}$, $\mathbf{W}_2 \in \mathbb{R}^{d_{ff} \times d}$, and $d_{ff}$ is the dimensionality of the hidden layer. The nonlinear transformation added by the FFN component further increases the capacity of the model and projects the embeddings to a higher dimension before lowering the dimension again to $d$.

To stabilize training and improve convergence, both the multi-head attention and FFN layers output are passed through a layer normalization operation following the work of \cite{ba2016layer}, and residual connections are added to promote gradient flow.  Formally, the output of the $l$-th encoder layer is computed as:
\begin{align*}
\mathbf{h}_i^{(l)} & = \text{LayerNorm} \left( \mathbf{h}_i^{(l-1)} + \text{MultiHead}(\mathbf{h}_i^{(l-1)}) \right), \\
\mathbf{h}_i^{(l+1)}& = \text{LayerNorm} \left( \mathbf{h}_i^{(l)} + \text{FFN}(\mathbf{h}_i^{(l)}) \right).
\end{align*}
These residual connections improve the model’s ability to propagate gradients during back-propagation, making deeper models easier to train.
\subsection{Graphical Representations of the Spatial Data and the Neural Bayes Estimator Network}
In this section, we visually illustrate the graphical representation of the spatial data in detail and demonstrate the architecture of the neural Bayes estimator for the GSUN  spatial process.
\subsubsection{Spatial Data} \label{data_rep}
Figure \ref{fig:visual_data} depicts 2-D spatial data denoted with an undirected graph ${\cal G} = ({\cal V}, {\cal \xi})$. This graphical representation is obtained by plotting the locations of the spatial data on the domain of interest and treating each location as a node; we connect two nodes (locations) if their distance is within a pre-specified radius, $R$.  Here, each node embedding ${ \textbf{E}_i} = (z_i,x_i,y_i)^\top$ with $i = 1,\dots,n$, where $z_i$ is a realization of the spatial process at location $(x_i,y_i) \in [0,1]^2$. When training the neural Bayes 
\begin{figure}[t!]
    \centering
    \includegraphics[width=0.4\linewidth]{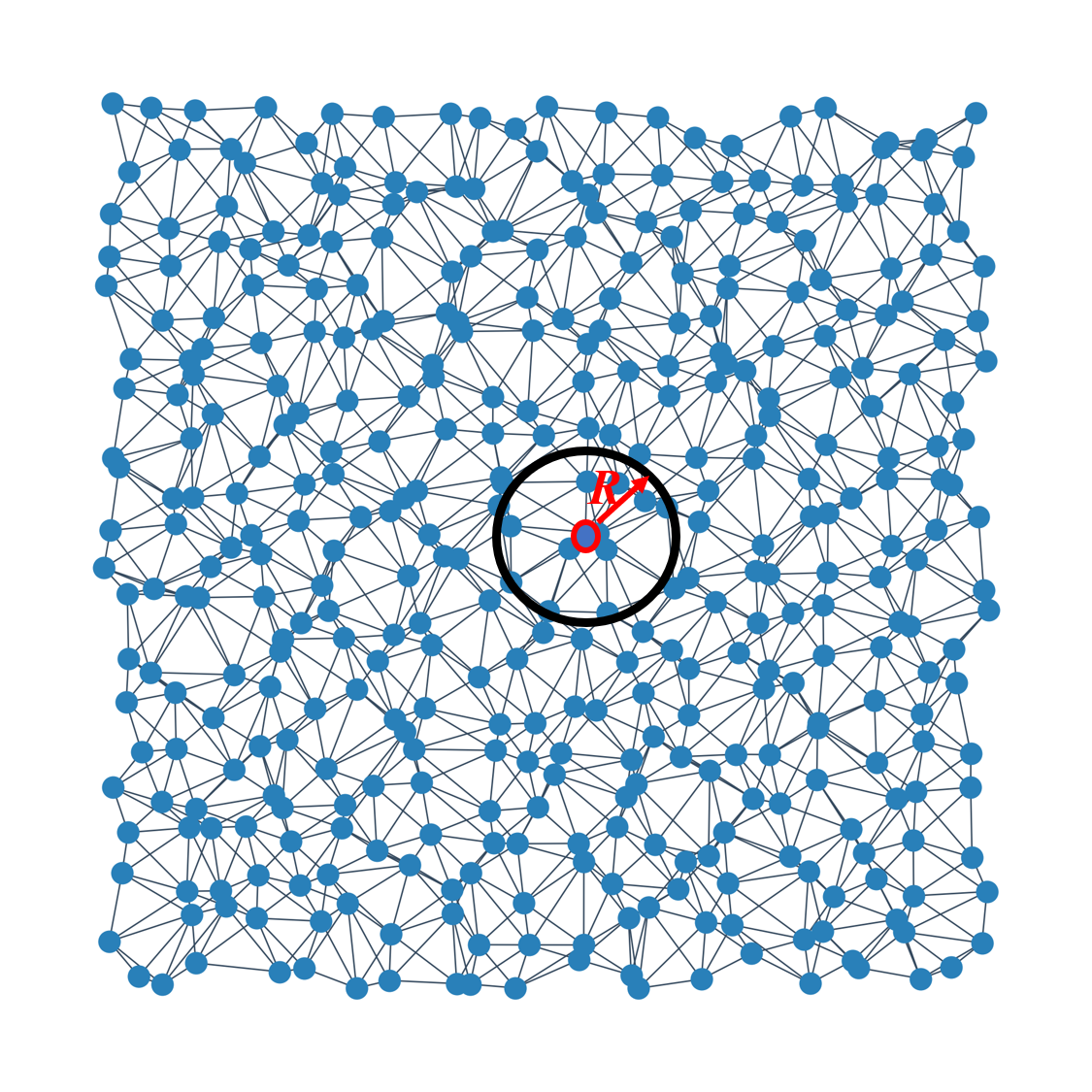}
    \caption{Graphical representation of the spatial data with undirected edges. The circle in black denotes a pre-specified radius $R$ to define the neighboring nodes. The red arrow denotes the radius, $R$.}
    \label{fig:visual_data}
\end{figure}
estimator, we set $R = 0.34$ so that the GATs are able to reach a sufficient coverage of the study region within 3 steps of aggregation (3 layers). Lastly, the edges ${\cal \xi}$ represent connectivity only because the spatial information has already been contained in the node-level embeddings and formatted into the input of the neural Bayes estimator, see Figure \ref{fig:trans-GAT}.
\subsubsection{Network Architecture}
 In this section, we visualize the architecture of the neural Bayes estimator in Figure \ref{fig:trans-GAT}. First, we format the input as two parts, ${\cal G}$ and $\textbf{D}$, where ${\cal G}$ is the graphical representation (defined in Section \ref{data_rep}) of the spatial data and $\textbf{D} \in \R^{n \times n}$ is the distance matrix for their corresponding locations. There are two paths of forward propogation in this network. First, ${\cal G}$ is passed through 3 layers of GATs, where each layer contains 8 attention heads ($n$-head = 8) and outputs the node-level embeddings as $32$, $256$, and $512$ dimensional vectors, respectively. Here, we set the output dimensions to powers of 2 for maximum 
 \begin{figure}[t!]
    \centering
    \includegraphics[width=0.8\linewidth]{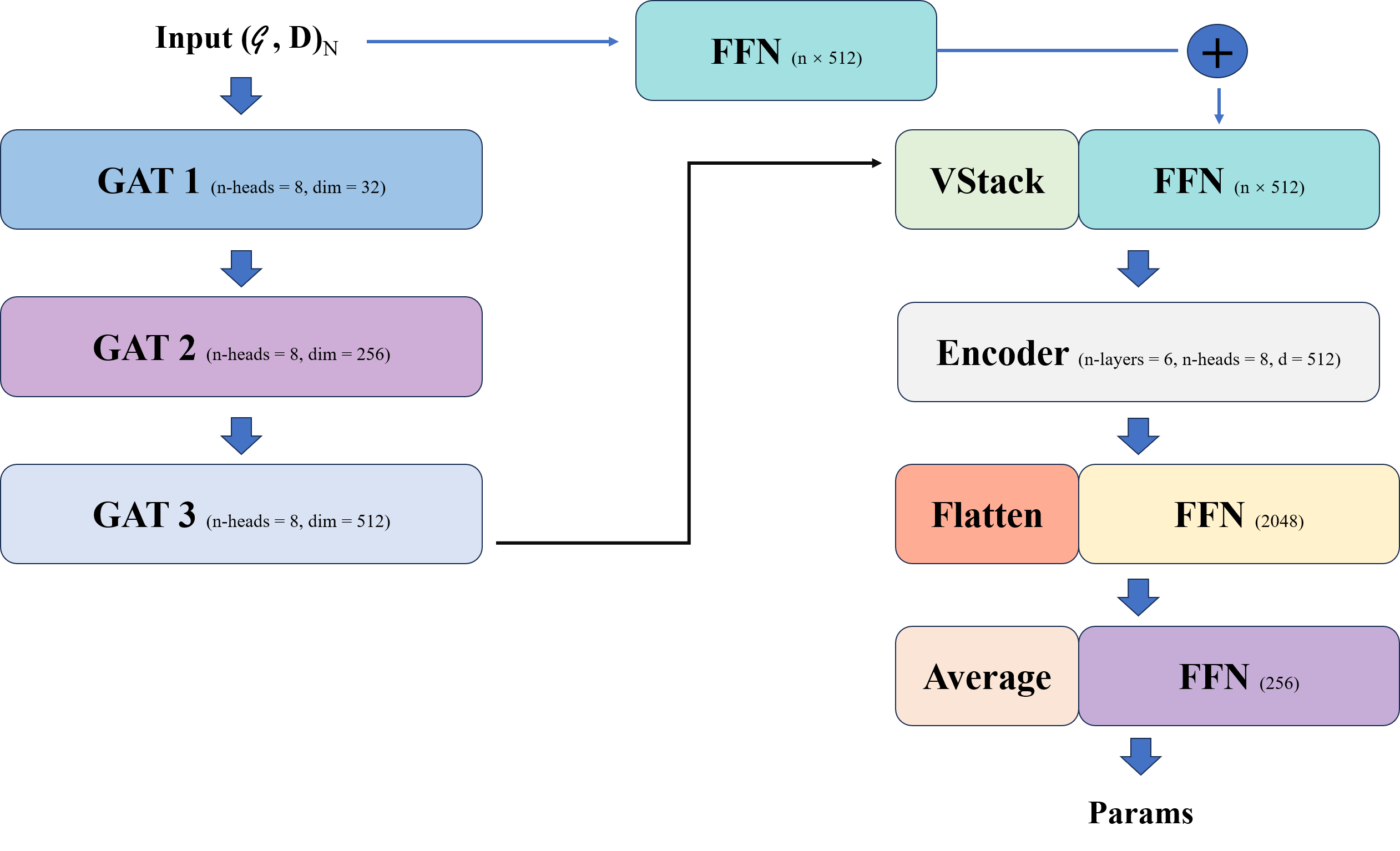}
    \caption{GAT- and Encoder-based neural Bayes estimator for the GSUN spatial process. $n$-heads denotes the number of attention heads for the multi-head attention mechanism. dim demonstrates the output dimension of each GAT layer. $n$-layers represents the number of encoder blocks. $d$ is the dimension of the input embeddings of the encoder block. The numbers in the parenthesis for each FFN layer illustrate its output dimension.}
    \label{fig:trans-GAT}
\end{figure}
General Processing Unit (GPU) capacity at the training stage. Then, the node-level embeddings are stacked in a column-wise manner in the \say{VStack} block.  Here, the dimensionally transformed distance matrix $\textbf{D}$ is added to the stacked embeddings. The network, then, passes the aggregated data to a FFN layer to be prepared as the input of the encoder block. As mentioned in Section \ref{encoder}, the encoder block, consisting of 6 encoders with 8 attention heads for each and 512 for the dimension of the input embeddings $d$, will further capture and extract dependencies within the data.  Moreover, the output is then collapsed into a 1-D vector and passed through a FFN layer for the last prediction layer.  Because there are $N$ replicates of $({\cal G}, \textbf{D})$, the network will process each replicate in parallel, { a concept that has been defined as $deepsets$ in \cite{zaheer2017deep}}. Hence, before going through the final prediction layer, the data are averaged across the batch dimension. Finally, the prediction layer maps the data to the parameter estimates $\hat\bTheta$. We set $d$ and the output sizes of all FFN layers as powers of 2 for maximum GPU capacity and they are tuned according to the convergence efficiency during training. The hyper-parameter tuning of the network architecture is demonstrated in Figure~\ref{training}. 
\begin{figure}[t!]
\centering
\begin{subfigure}{0.49\textwidth}
  \centering
  \includegraphics[width=1\textwidth,]{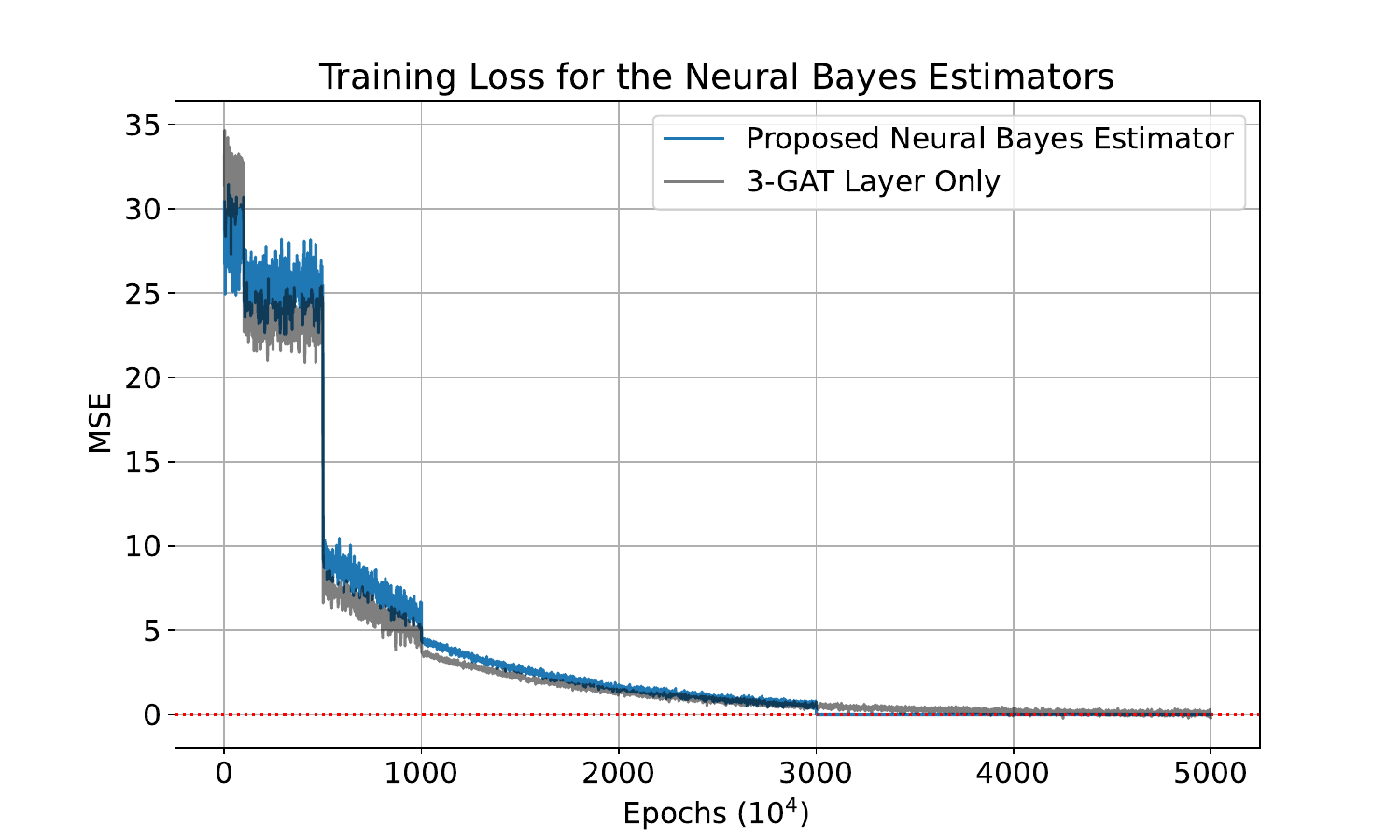}
  \caption{Training loss for the proposed neural Bayes estimator with and without an encoder block.} 
  \label{training_loss_trans}
\end{subfigure}
\begin{subfigure}{0.49\textwidth}
  \centering
  \includegraphics[width=1\textwidth,]{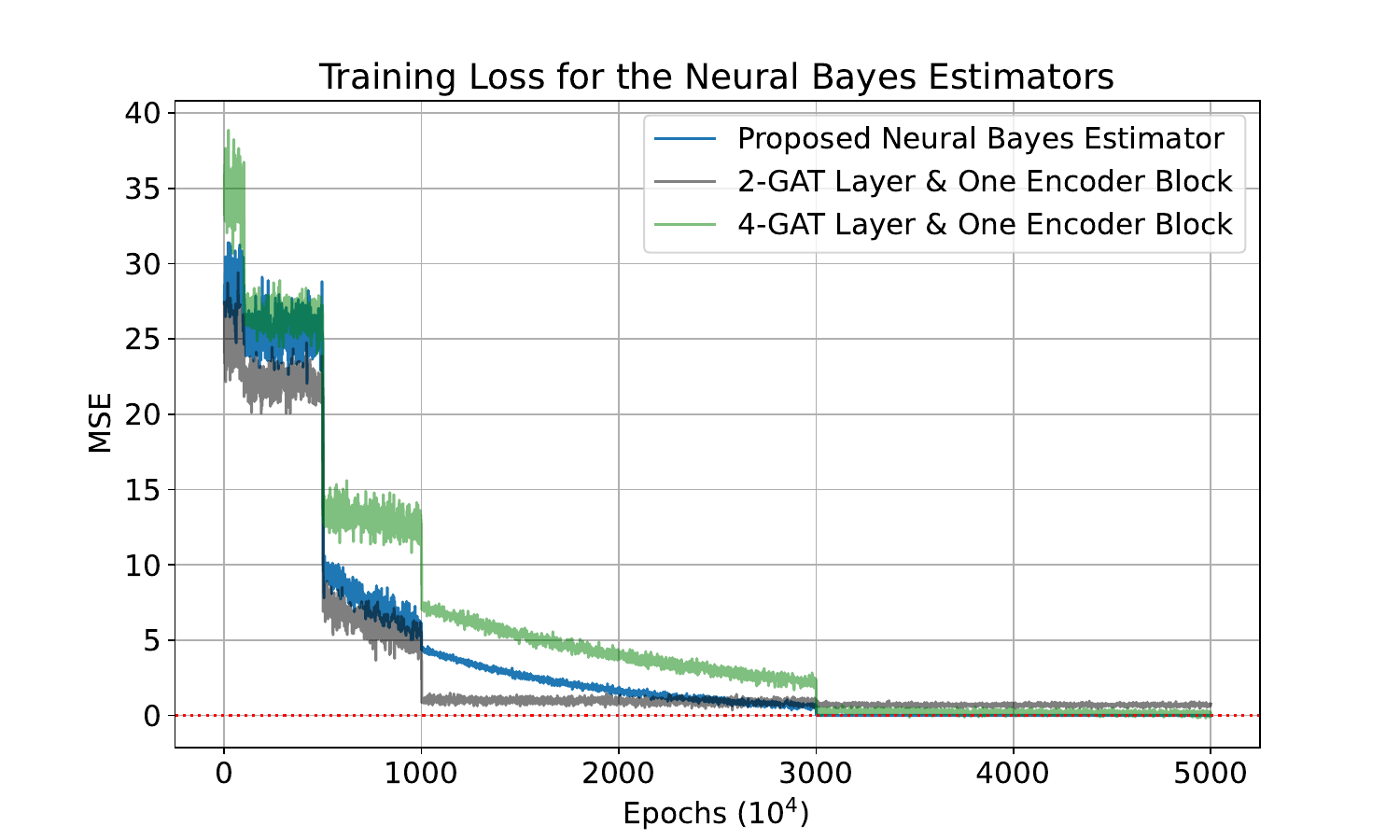}
  \caption{Training loss curves for the proposed neural Bayes estimators with various GAT layers.  }
  \label{training_loss_GAT}
\end{subfigure}
\caption{Hyper-parameter tuning for the proposed neural Bayes estimator. The red line denotes the threshold $10^{-5}$. The learning rates start as $10^{-3}$ and are manually adjusted (multiplied by 0.1) at $(100,500,1000,3000) \times 10^4$ epochs. The radius $R$ is adapted to the differing numbers of GAT layers such that the aggregations can sufficiently cover the study region within the specified number of layers.}
\label{training}
\end{figure}
Figure~\ref{training_loss_trans} shows that having an encoder transformer block indeed helps the Bayes estimator to converge, which can be attributed to the enhanced modeling capacity and more profound understandings of the inherent dependence structure. Otherwise, the estimator experiences relatively large volatility even with $5\times 10^7$ simulated training data. In detail, Figure \ref{training_loss_GAT} has demonstrated that having more layers of a particular architecture does not always lead to better efficiency. Redundant layers can cause the neural Bayes estimator to converge at a slower rate (more simulations needed to be properly trained).
\section{Simulation Study} \label{simulation}
In this section, we illustrate and analyze the performance of our proposed neural Bayes estimator for the GSUN spatial process. First, we compare the inference performance of our network with a CNN-based network architecture, which has been used for ammotized inference in \cite{goodfellow2016deep} and \cite{sainsbury2024likelihood} for Gaussian processes. Second, we show that the GSUN process has its own distinct features through the probability integral transform (PIT)  when compared to Gaussian and Tukey $g$-and-$h$ \citep{xu2017tukey} processes. Lastly, we demonstrate the uncertainty quantification of our neural Bayes estimator and verify its validity.  
\subsection{Analysis of inference results} \label{train_simulation}
We compare the inference results of our neural Bayes estimator with the inference results obtained from a CNN-based estimator, see Figure \ref{arc_neural_bayes} in Section S.2 of the Supplementary Materials. In addition, Section S.2 also includes a detailed description and explanation of the network structure and the associated ammortized inference. Both Bayes estimators are sufficiently and identically trained such that their empirical Bayes risks fall below $10^{-5}$ for a sufficiently large number of consecutive simulations, see Figure \ref{training_loss} in the Supplementary Materials. 
\begin{figure}[b!]
\centering
\begin{subfigure}{0.32\textwidth}
  \centering
  \includegraphics[width=1\textwidth,]{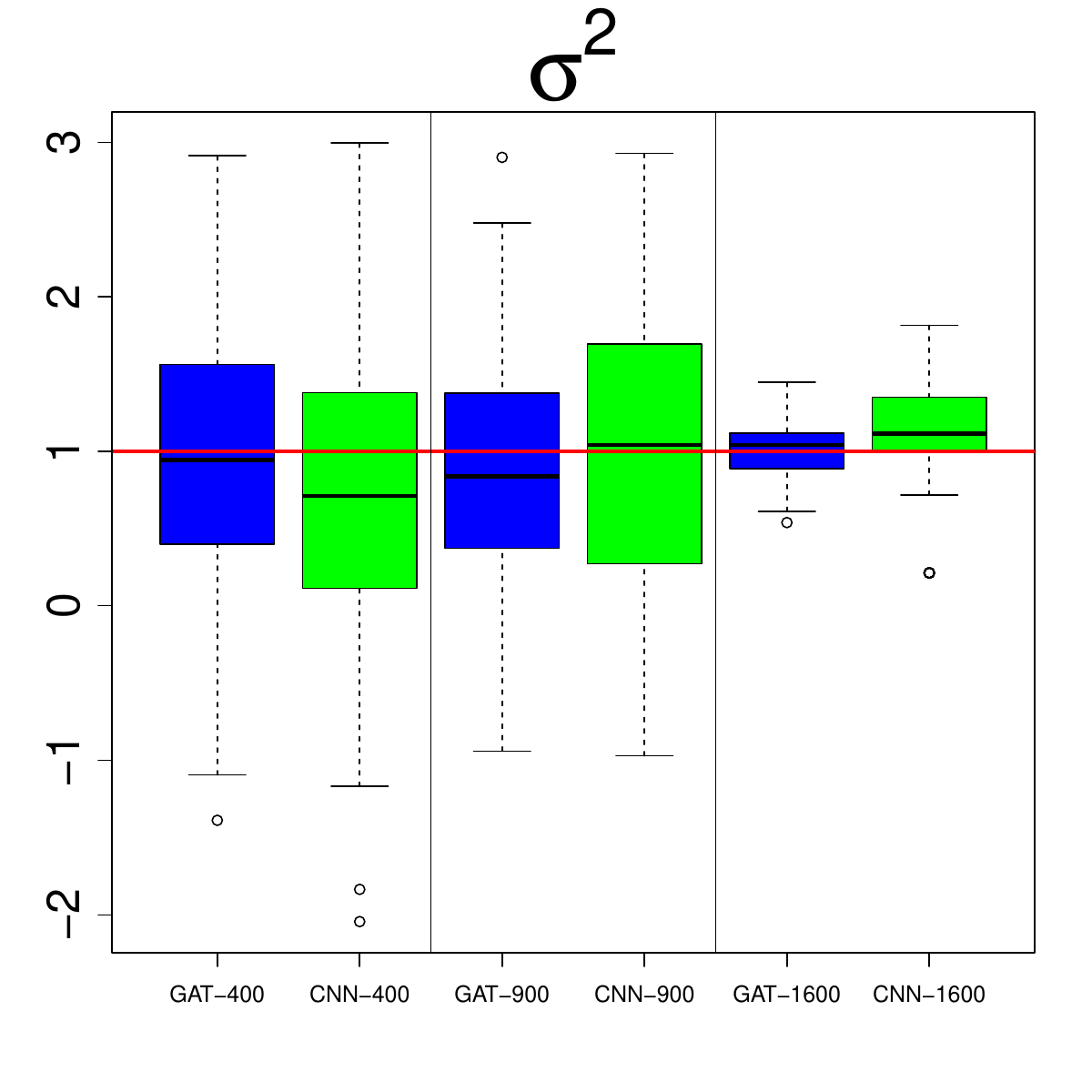}
\end{subfigure}
\begin{subfigure}{0.32\textwidth}
  \centering
  \includegraphics[width=1\textwidth,]{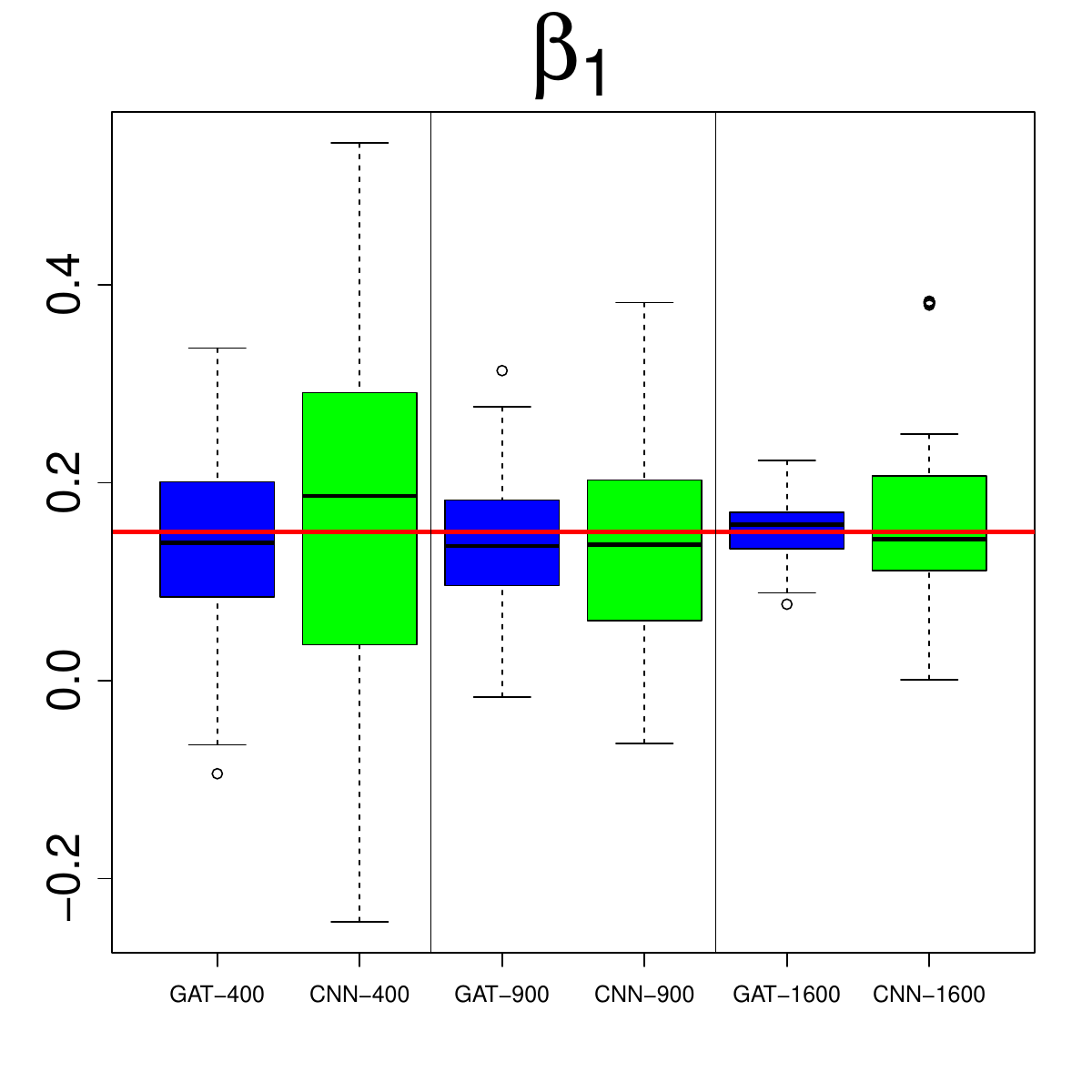}
\end{subfigure}
\begin{subfigure}{0.32\textwidth}
    \centering
    \includegraphics[width = 1\textwidth,]{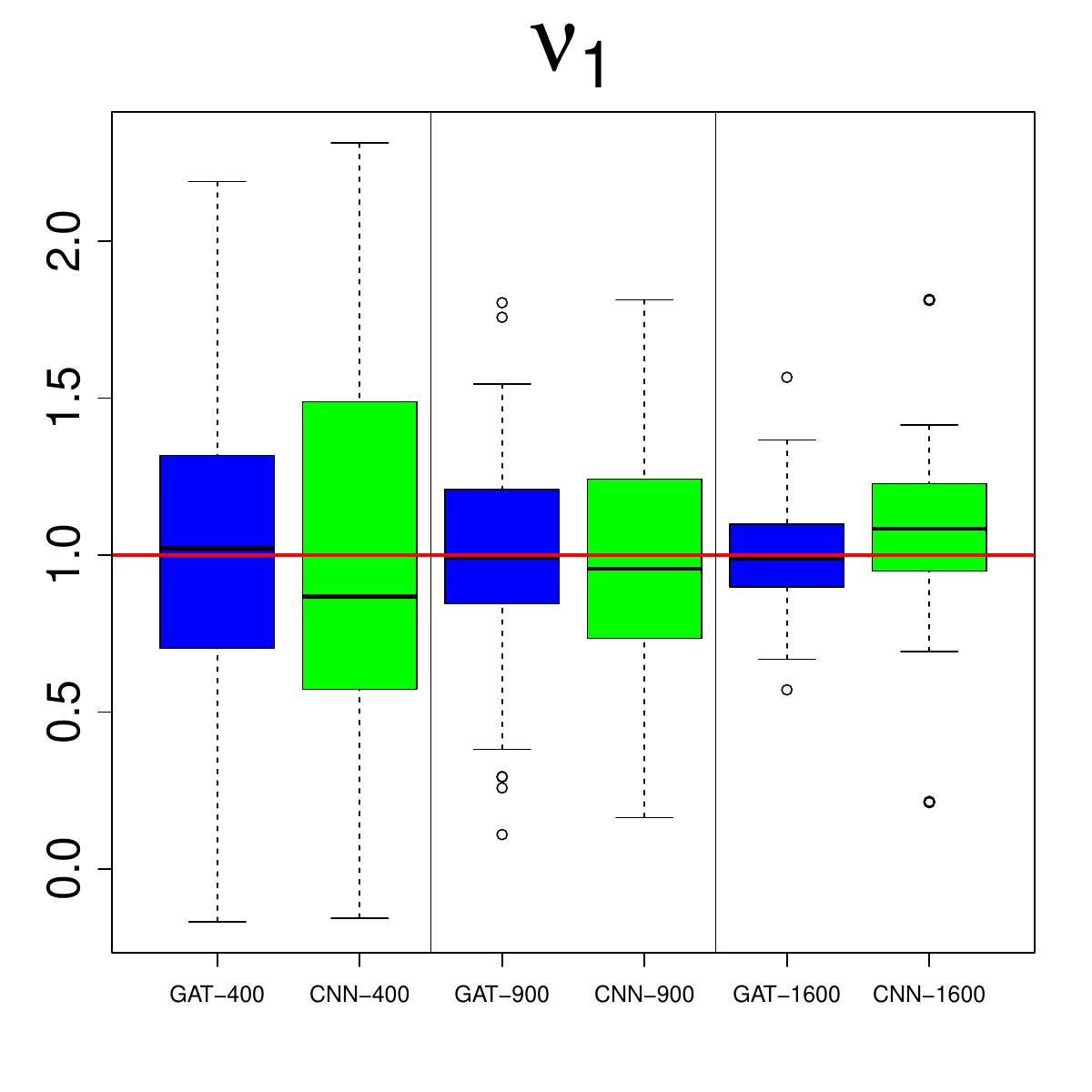}
\end{subfigure}
\begin{subfigure}{0.32\textwidth}
    \centering
    \includegraphics[width = 1\textwidth,]{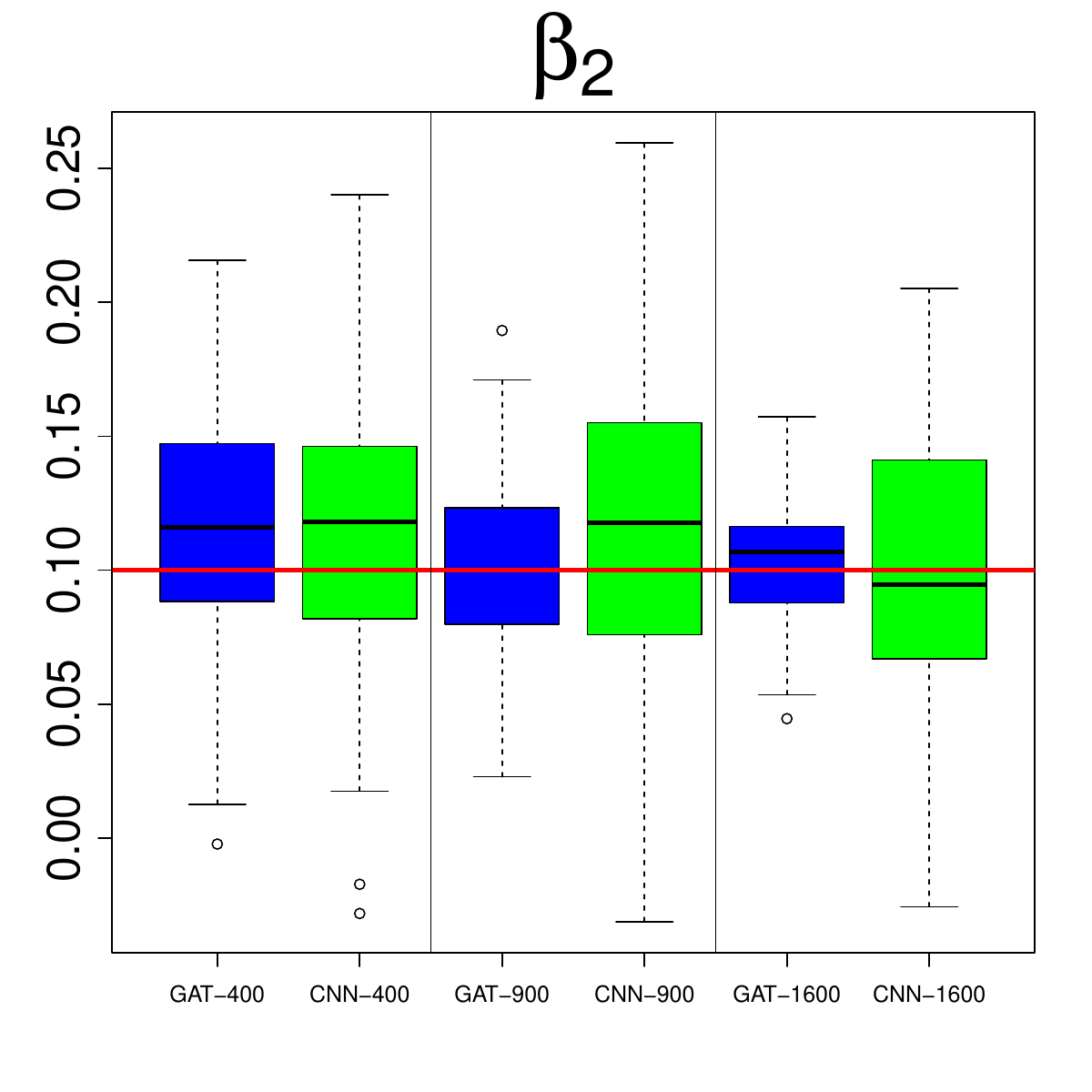}
\end{subfigure}
\begin{subfigure}{0.32\textwidth}
    \centering
    \includegraphics[width = 1\textwidth,]{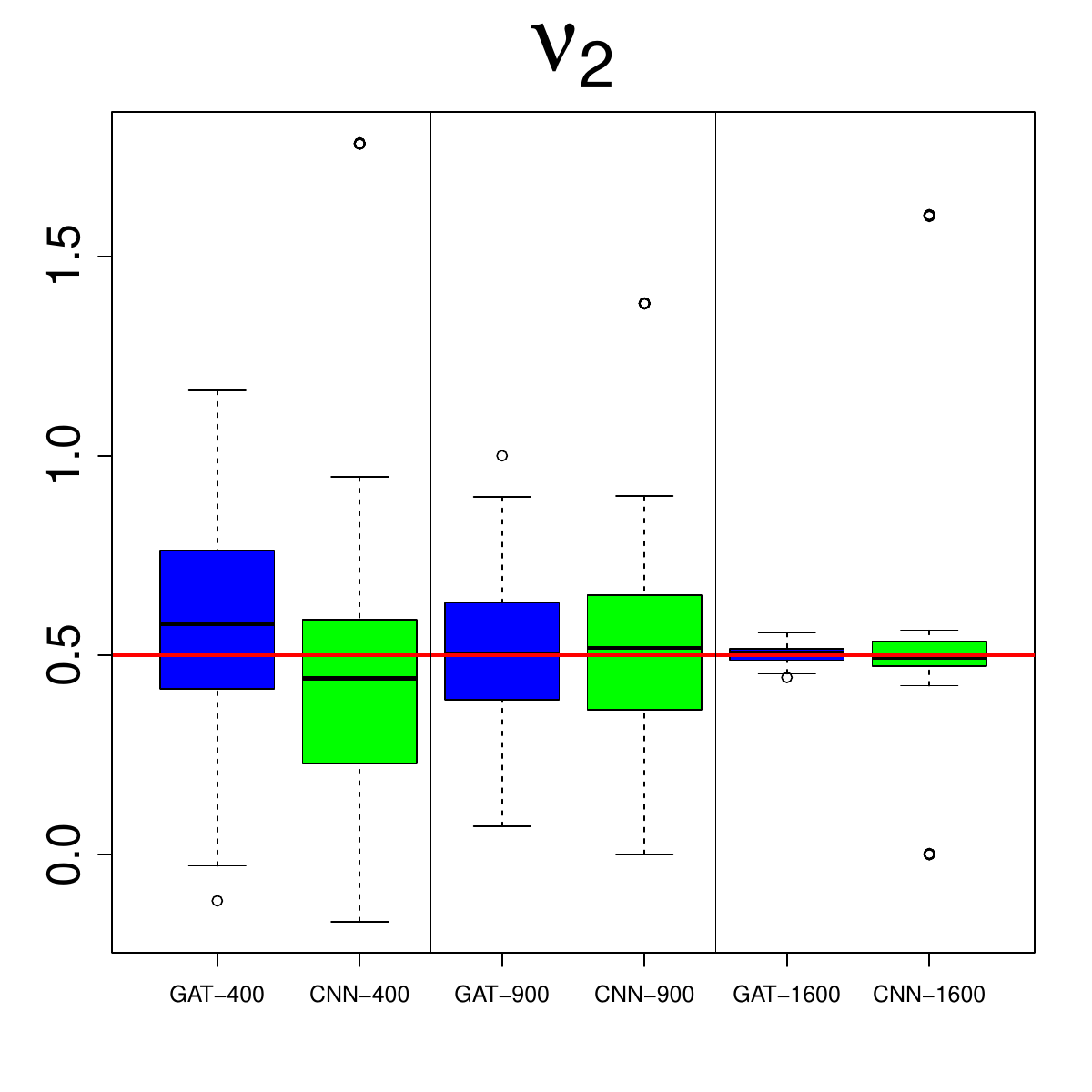}
\end{subfigure}
\begin{subfigure}{0.32\textwidth}
    \centering
    \includegraphics[width = 1\textwidth,]{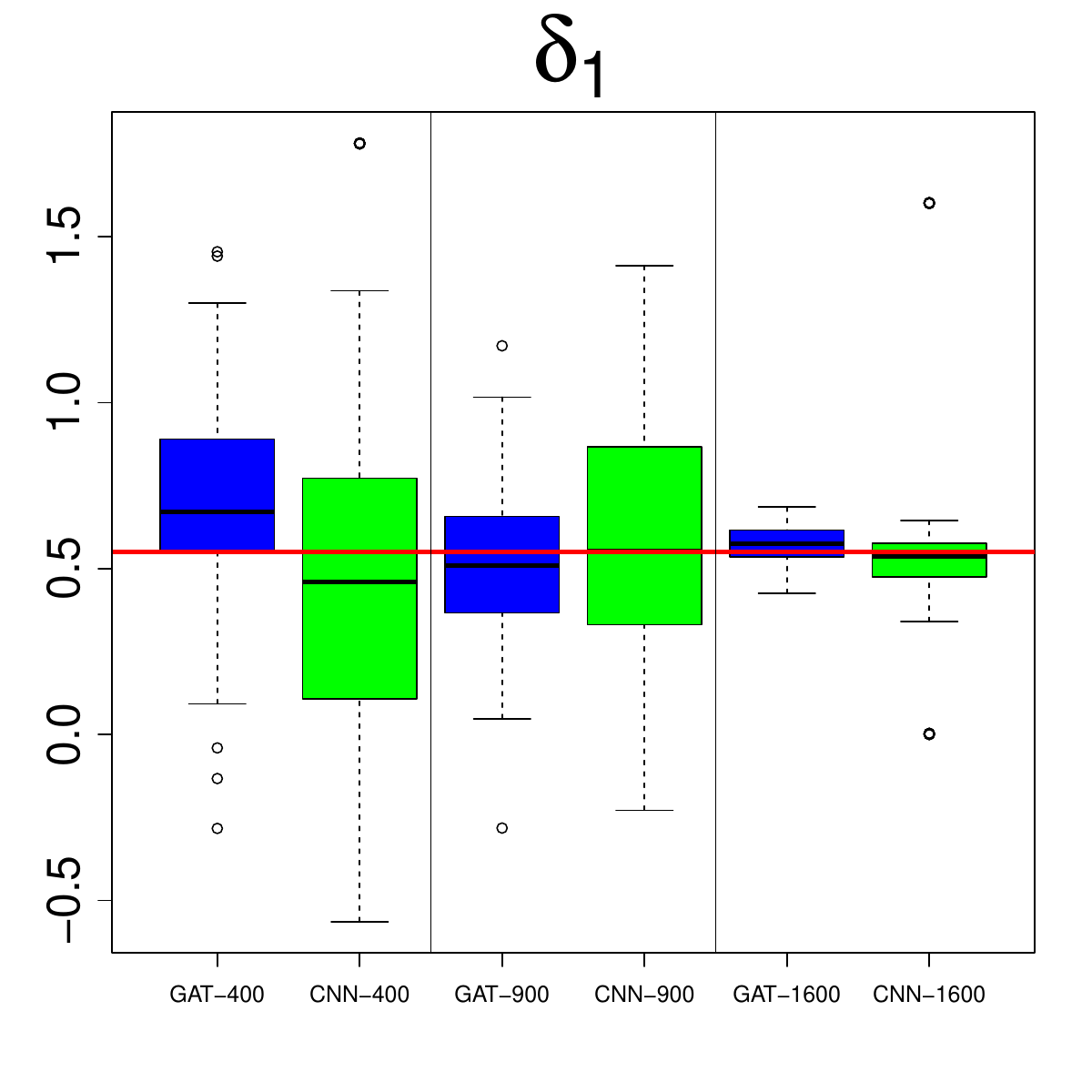}
\end{subfigure}
\begin{subfigure}{0.32\textwidth}
    \centering
    \includegraphics[width = 1\textwidth,]{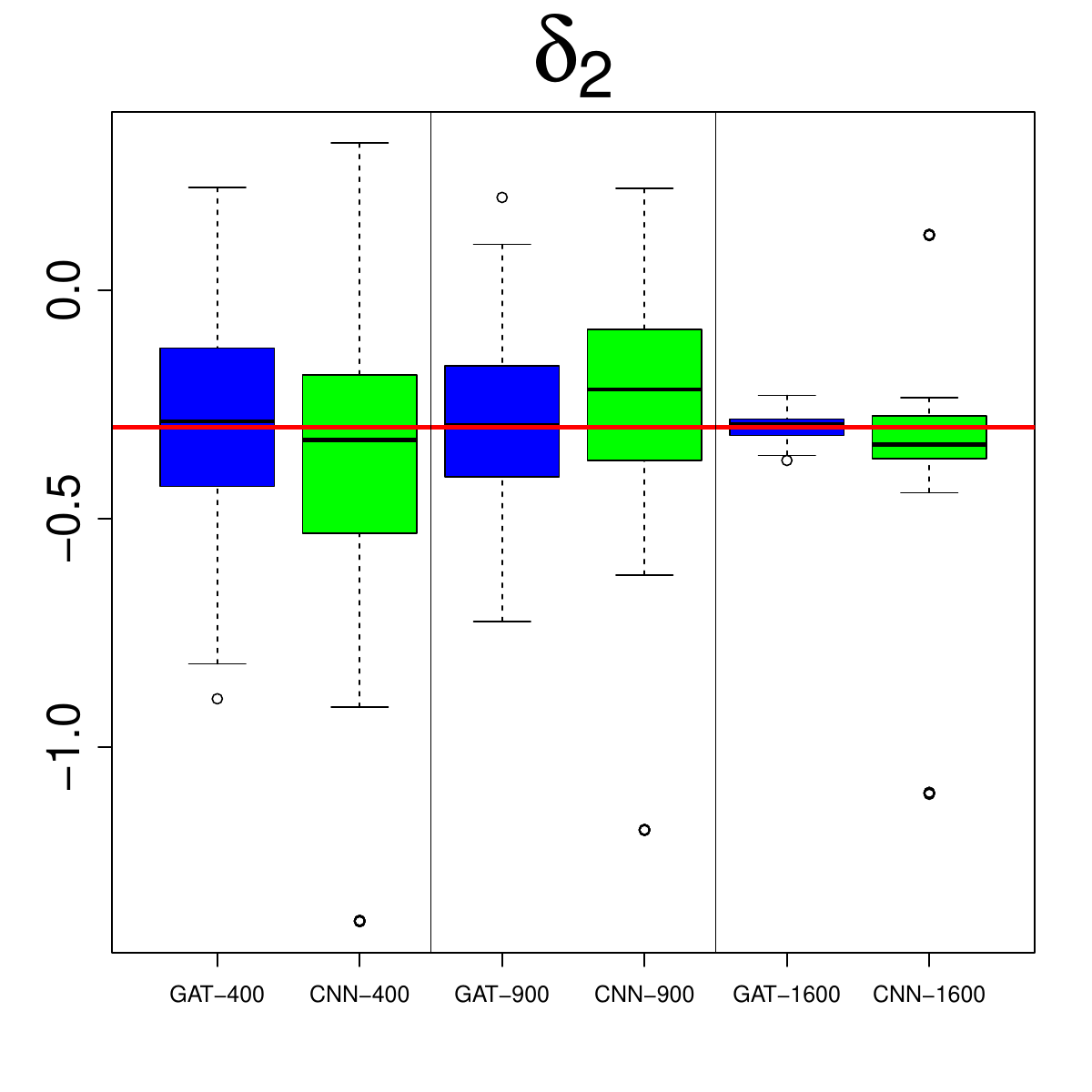}
\end{subfigure}
\begin{subfigure}{0.32\textwidth}
    \centering
    \includegraphics[width = 1\textwidth,]{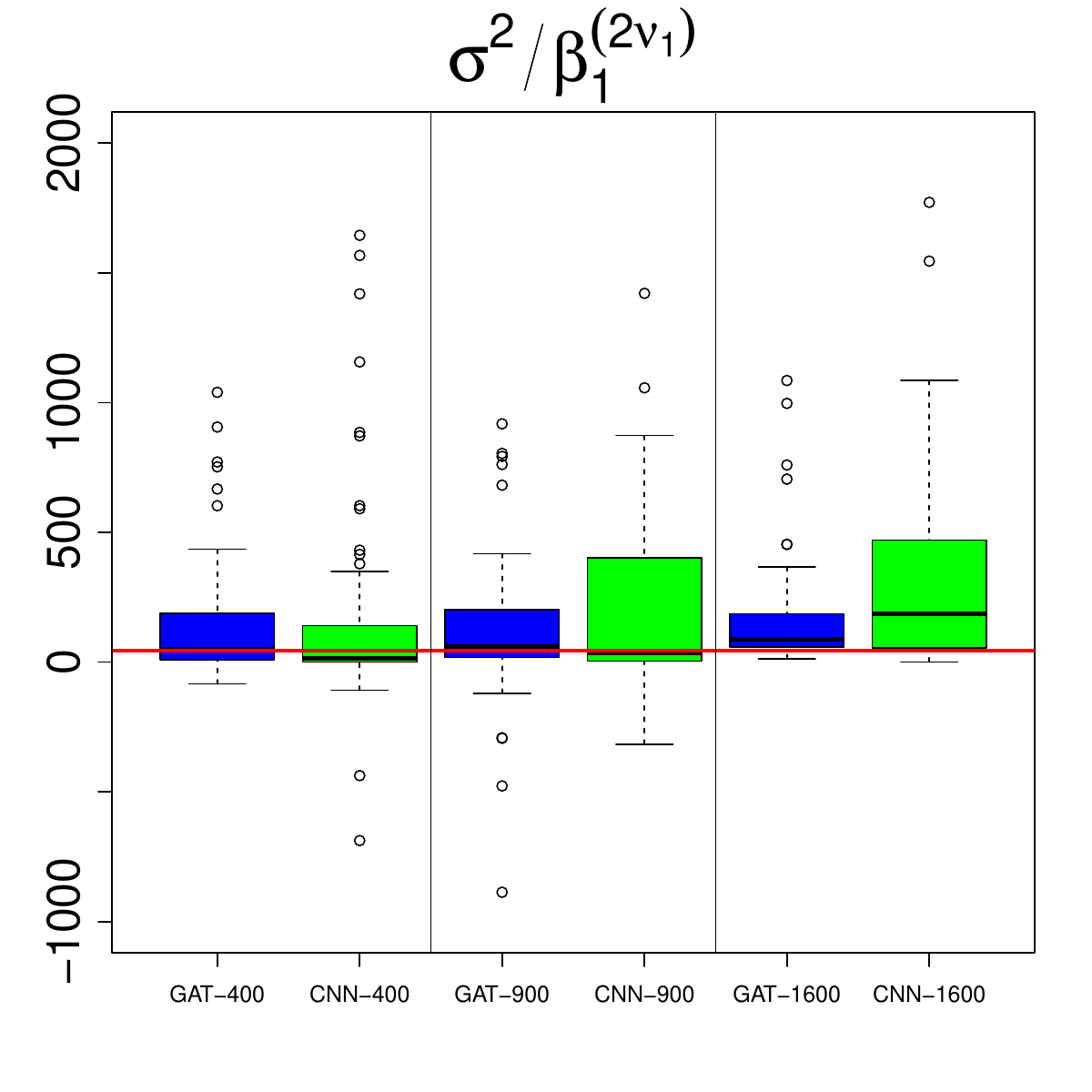}
\end{subfigure}
\begin{subfigure}{0.32\textwidth}
    \centering
    \includegraphics[width = 1\textwidth,]{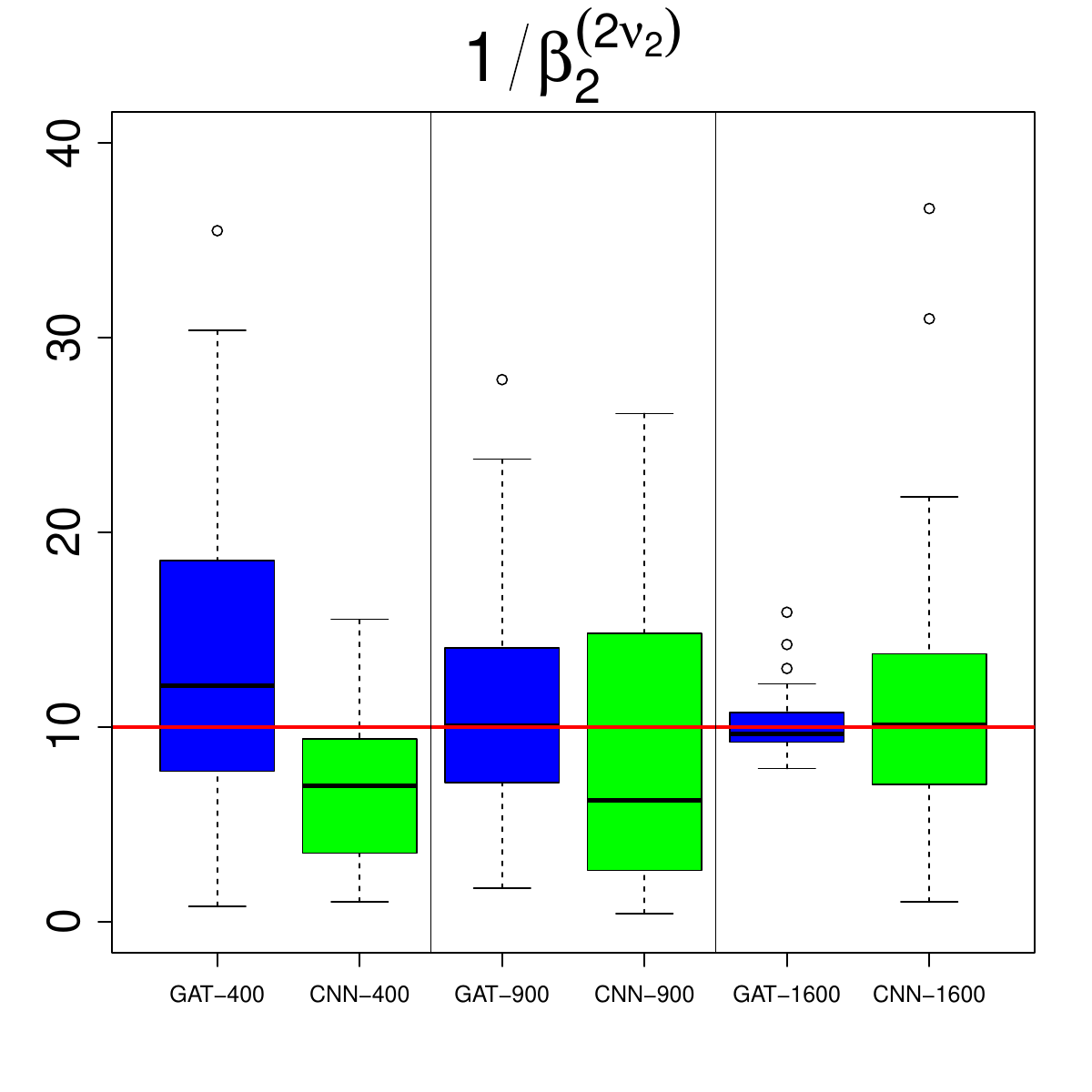}
\end{subfigure}
\caption{The boxplots of the parameter estimates obtained from the proposed neural Bayes estimator (marked in Blue and denoted as GAT-$n$) and the CNN-based Bayes estimator (marked in Green and denoted as CNN-$n$) for $N = 500$ replicates of $n = 400,900, 1600$ realizations of the GSUN process in $[0,1]^2$ simulated from $\bTheta = (1, 0.15, 1, 0.1, 0.5, 0.55, -0.3)^\top$. The red lines denote the true values.}
\label{inference}
\end{figure}

To train the neural Bayes estimator, we simulate $\sigma^2 \sim {\cal U}(0.3,3)$, $\beta_1 \sim {\cal U}(0.01,1)$, $\nu_1 \sim {\cal U}(0.3,2)$, $\beta_2 \sim {\cal U}(0.01,1)$, $\nu_2 \sim {\cal U}(0.3,2)$, $\delta_1 \sim {\cal U}(-3,3)$, and $\delta_2 \sim {\cal U}(-3,3)$, where ${\cal U}(l,u)$ denotes a uniform distribution on the interval $(l,u)$. We choose these intervals to maintain a reasonable coverage of the ranges of the true values while avoiding exploding the number of samples needed from simulation to train the network. 

When $\delta_1 = 0$ and $\delta_2 = 0$, the GSUN spatial model reduces to the conventional Gaussian process. As a result, the range ($\beta_2$) and smoothness ($\nu_2$) parameters for the latent truncated process become non-estimable. Therefore, we remove them when calculating $L(\bTheta^n,\hat\bTheta_{(\bz^n,\bbeta)})$, which is chosen as the Mean Square Error (MSE) loss. During the training process, we relax the condition for such situations to avoid weakly estimable cases, low-quality data that can potentially pollute the neural network, by removing $\beta_2$ and $\nu_2$ from the loss when $(\delta_1,\delta_2) \in [-0.1,0.1]^2$.  

To obtain the inference results of the trained neural Bayes estimator, we simulate $N = 1$ replicate of $\bZ \sim {\cal SUN}_{n,n}(\0, \bSigma(\btheta_1),\bH,\0, \textbf{C}(\btheta_2))$ with $n =400,900,1600$ and $\bTheta = (1,0.15,1,0.1,0.5,0.55,\\ -0.3)^\top$ from randomly generated locations. It is important to note that, here, we set the sample size $n = 100$ only for the purpose of training the estimator more efficiently. During inference, one can easily restructure the data into multiple sets of $n = 100$ samples and format these datasets as replicates for input. By doing this, the estimator is not limited to a fixed sample size $n = 100$.

In Figure~\ref{inference}, we can see the empirical posterior distribution $\Pi(\hat\bTheta|\bZ,\bbeta)$ for both Bayes estimators are centered around specific values close to the true values and become more concentrated as sample size increases.

In addition, the neural Bayes estimator proposed in this work outperforms the conventional CNN-based structure regarding accuracy and stability. Figure \ref{inference} illustrates that the empirical distributions of $\hat{\bTheta}$ computed using the proposed estimator has smaller IQRs across all parameters and the corresponding means are closer to the true values compared to those computed with the CNN-based estimator for most of the cases. Furthermore, $\sigma^2/\beta_1^{2\nu_1}$ and $1/\beta_2^{2\nu_2}$, computed using the proposed estimator, better comply with the theory of consistent estimation for in-domain asymptotics proposed in \cite{zhang2004inconsistent}. Such advantages can be attributed to the more expressive graphical representation of spatial data and the well-crafted architectures that allow for a more profound understanding
of the dependence structure and a stronger modeling capacity.
\subsection{Uncertainty Quantification}
In this section, we explore the uncertainty quantification method for our proposed neural Bayes estimator. In \cite{sainsbury2023neural}, the authors have mentioned that one way to quantify the uncertainty of the neural Bayes estimator is to train another neural Bayes estimator
\begin{algorithm}[b!]
\caption{Uncertainty Quantification of Neural Bayes Estimator}\label{alg:uc}
\textbf{Require}: Trained neural Bayes estimator $\hat\bbeta$, graphical spatial data $({\cal G},\textbf{D})$, number of bootstrap samples $j$, number of replicates $k$\\
\textbf{Procedure}:\\
\hspace*{5mm} (1) Take $j$ bootstrap samples on $({\cal G},\textbf{D})$ \\
\hspace*{5mm} (2) Input $({\cal G},\textbf{D})_i, \forall i = 1,\dots, j$ into $\hat\bbeta$ to obtain $\hat{\bTheta}^i$ \\
\hspace*{5mm} (3) Compute the average of $\hat{\bTheta}^i$ for $i = 1,\dots,j$ and denote it as $\bar\bTheta$ \\
\hspace*{5mm} (4) Simulate $k$ replicates of $\Tilde{\bZ} \sim f(\Tilde{\bZ}|\bar{\bTheta})$ and denote it as $\Tilde\bz^k$ \\
\hspace*{5mm} (5) Input $\Tilde \bz^k$ independently into $\bbeta$ to obtain $\hat{\bTheta}_{(\Tilde{\bz},\bbeta)}^k$ \\
\hspace*{5mm} (6)  Compute the empirical quantiles using $\hat{\bTheta}_{(\Tilde{\bz},\bbeta)}^k$ 
\end{algorithm}
that can learn various quantiles of the estimates based on the given sample realizations. Nonetheless, such method requires training numerous independent networks because it is difficult for a particular network to learn to output all quantiles (continuous) at once. Therefore, in this work, we present a simulation-based approach, which can accurately and efficiently approximate the quantiles and originate { from a similar idea to the bootstrap approach proposed in \cite{sainsbury2024likelihood}.} A detailed description is illustrated in Algorithm \ref{alg:uc}.

\begin{figure}[b!]
\centering
\begin{subfigure}{0.32\textwidth}
  \centering
  \includegraphics[width=1\textwidth,]{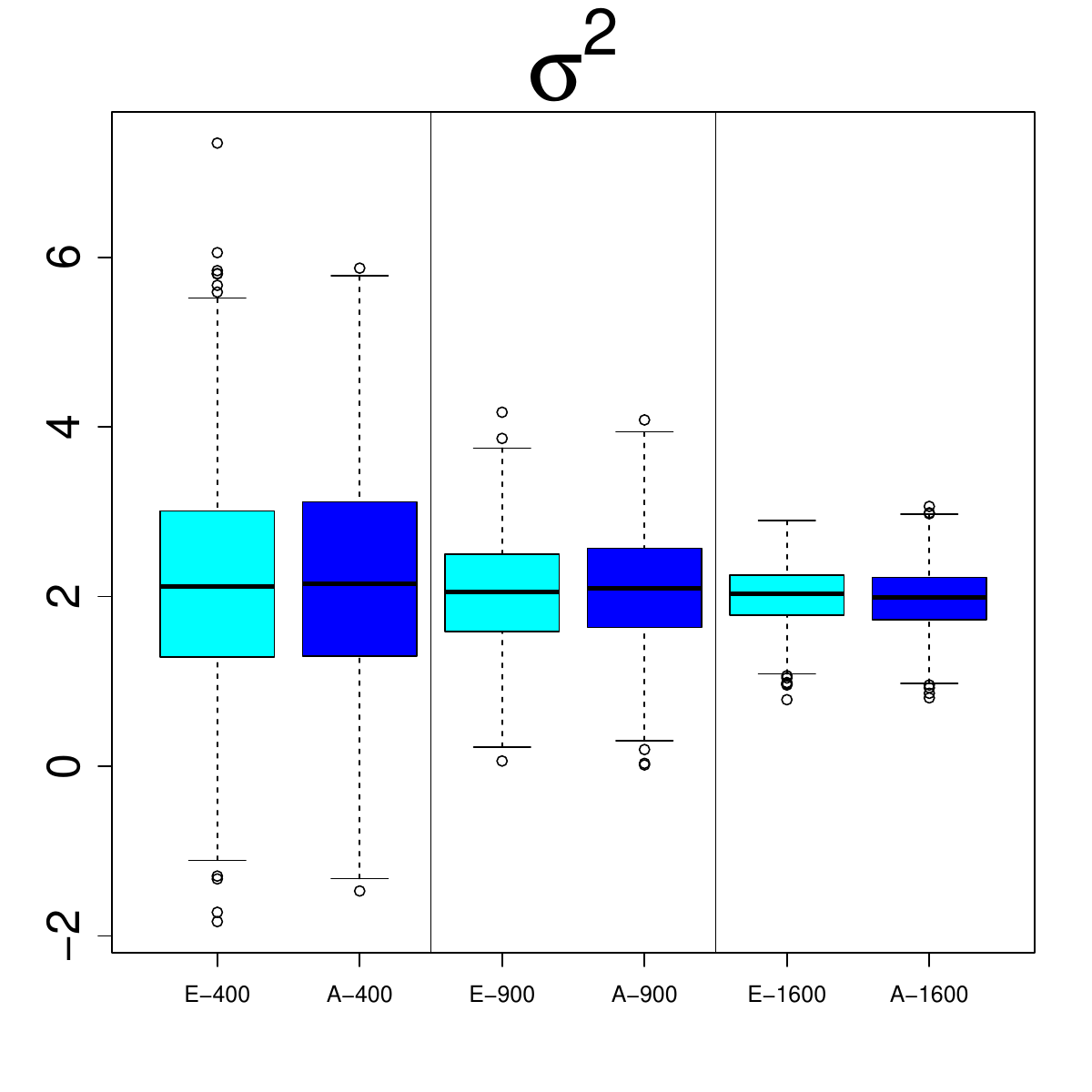}
\end{subfigure}
\begin{subfigure}{0.32\textwidth}
  \centering
  \includegraphics[width=1\textwidth,]{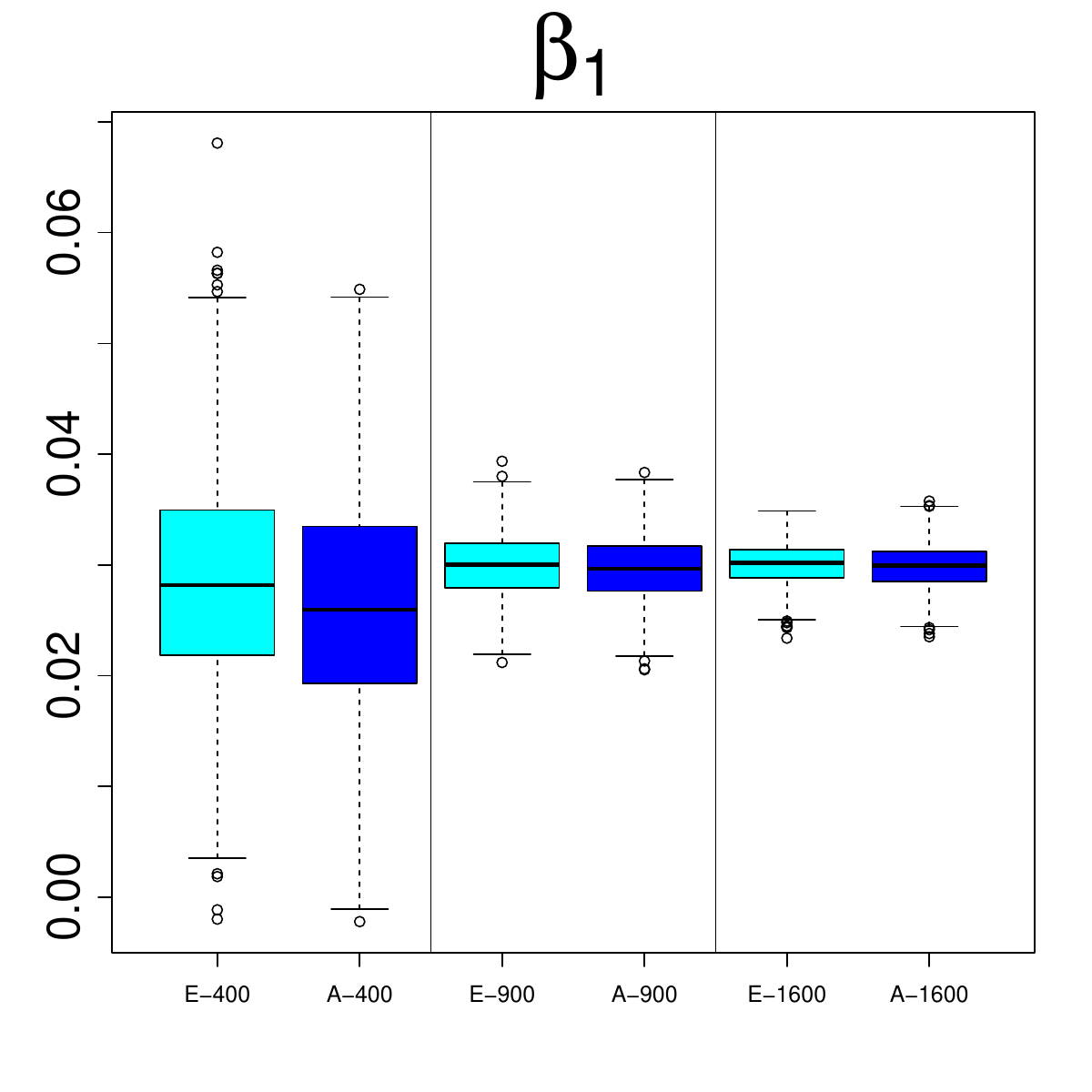}
\end{subfigure}
\begin{subfigure}{0.32\textwidth}
    \centering
    \includegraphics[width = 1\textwidth,]{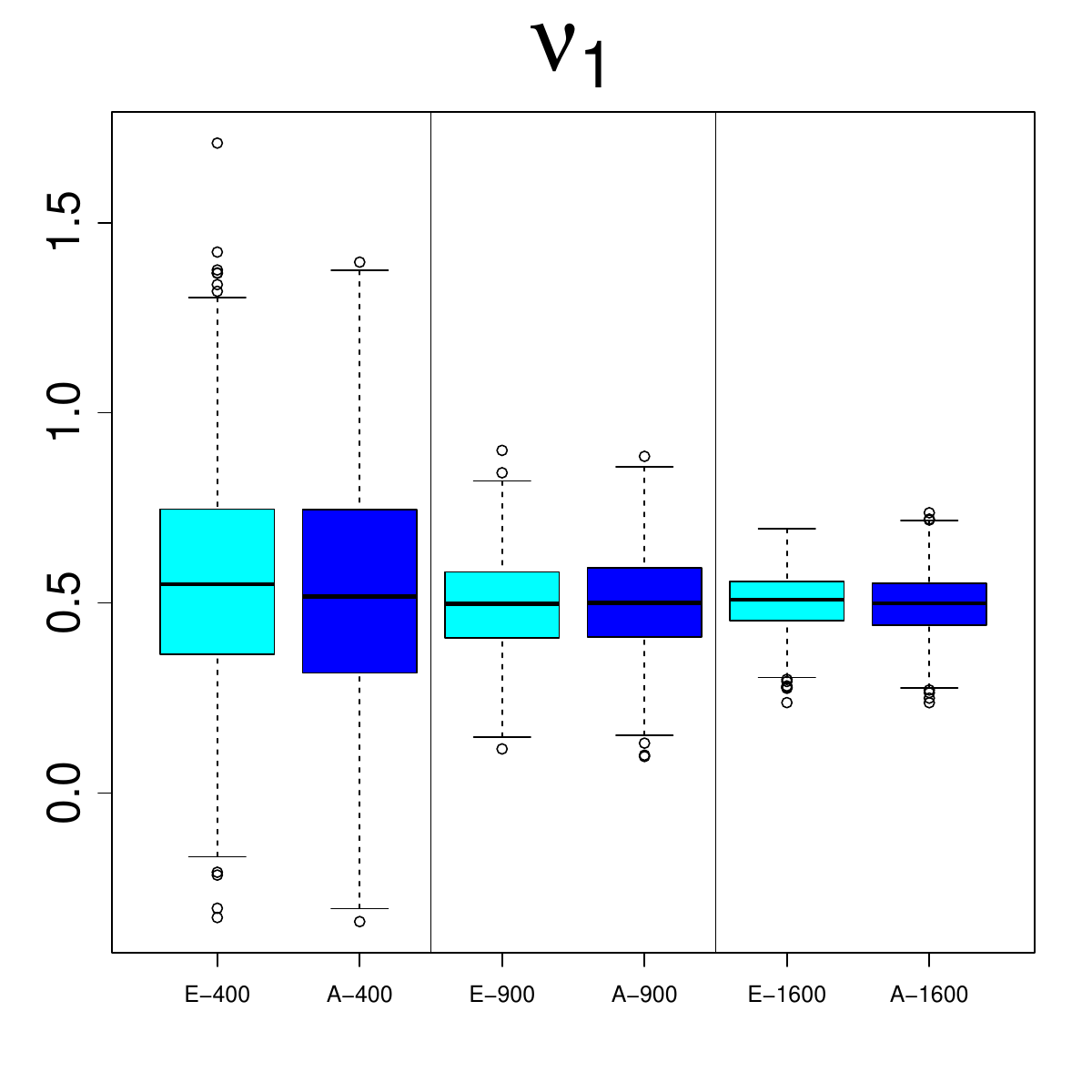}
\end{subfigure}
\begin{subfigure}{0.32\textwidth}
    \centering
    \includegraphics[width = 1\textwidth,]{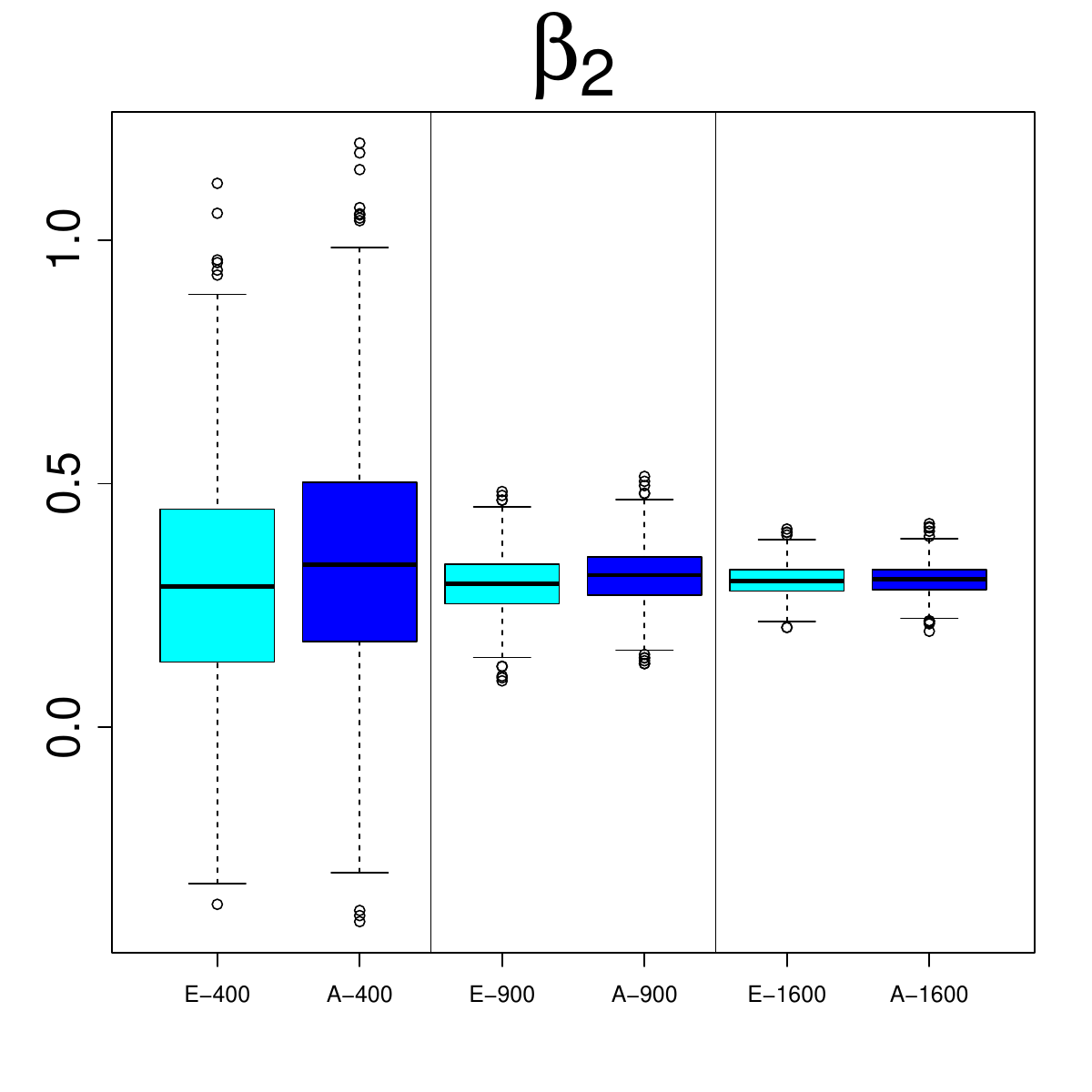}
\end{subfigure}
\begin{subfigure}{0.32\textwidth}
    \centering
    \includegraphics[width = 1\textwidth,]{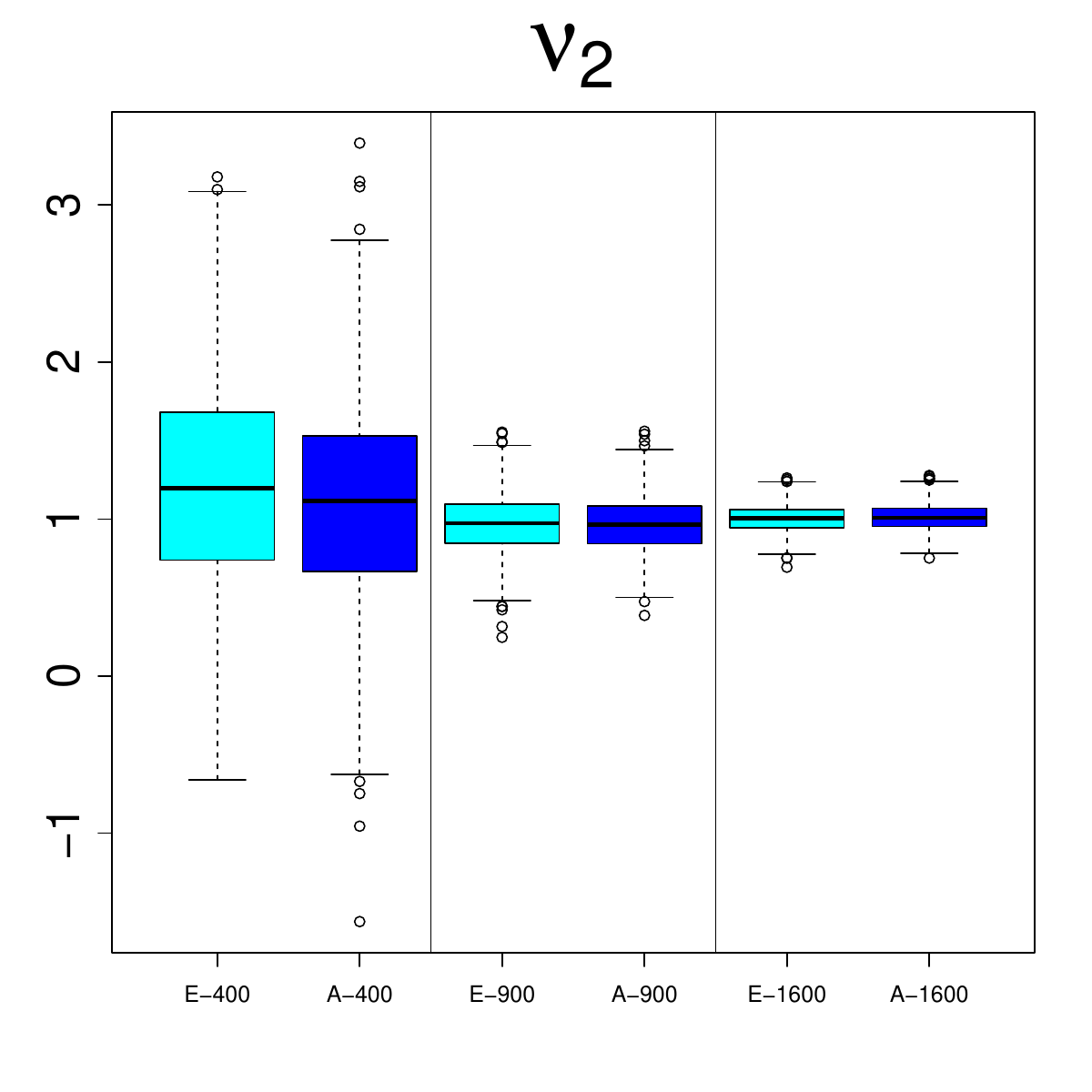}
\end{subfigure}
\begin{subfigure}{0.32\textwidth}
    \centering
    \includegraphics[width = 1\textwidth,]{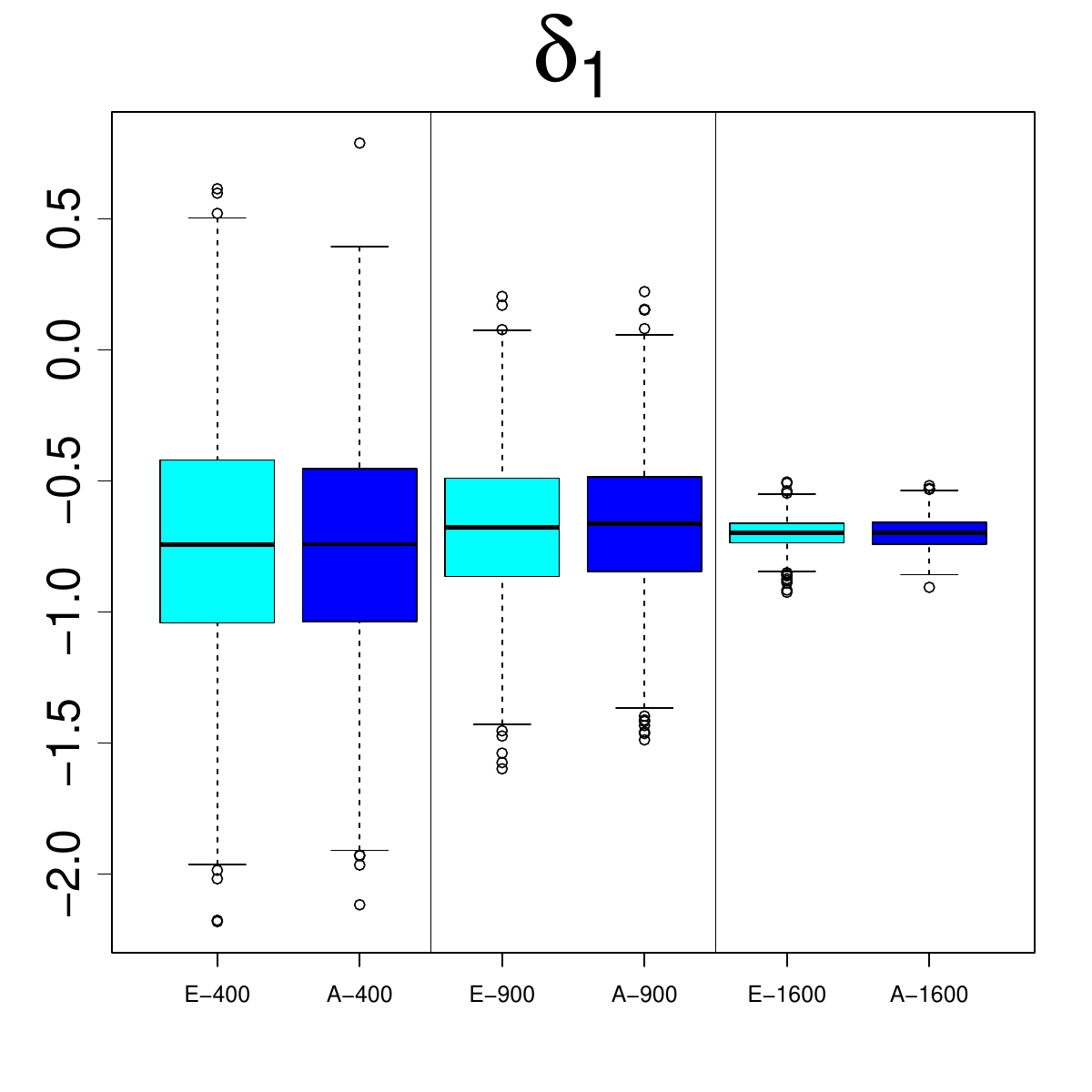}
\end{subfigure}
\begin{subfigure}{0.32\textwidth}
    \centering
    \includegraphics[width = 1\textwidth,]{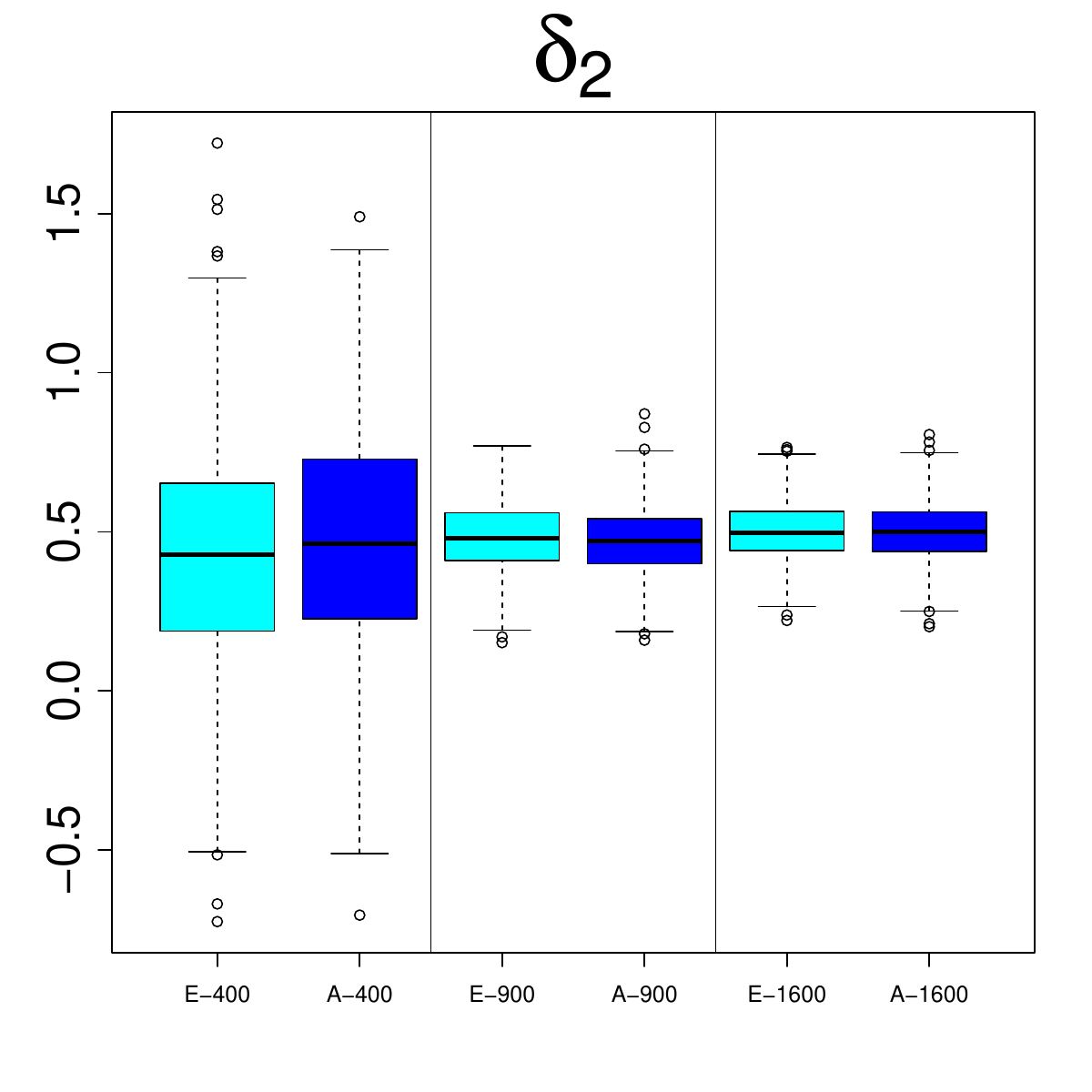}
\end{subfigure}
\begin{subfigure}{0.32\textwidth}
    \centering
    \includegraphics[width = 1\textwidth,]{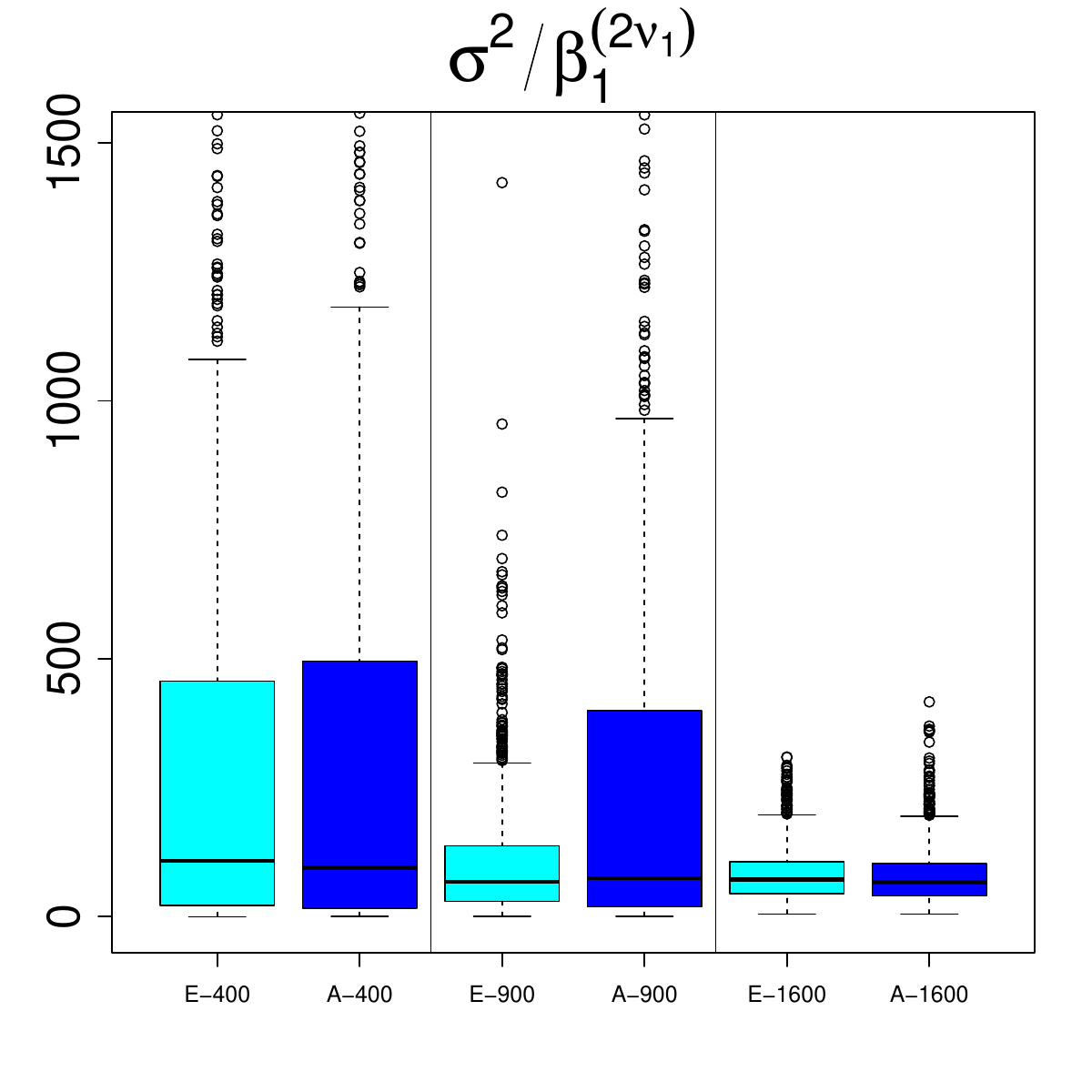}
\end{subfigure}
\begin{subfigure}{0.32\textwidth}
    \centering
    \includegraphics[width = 1\textwidth,]{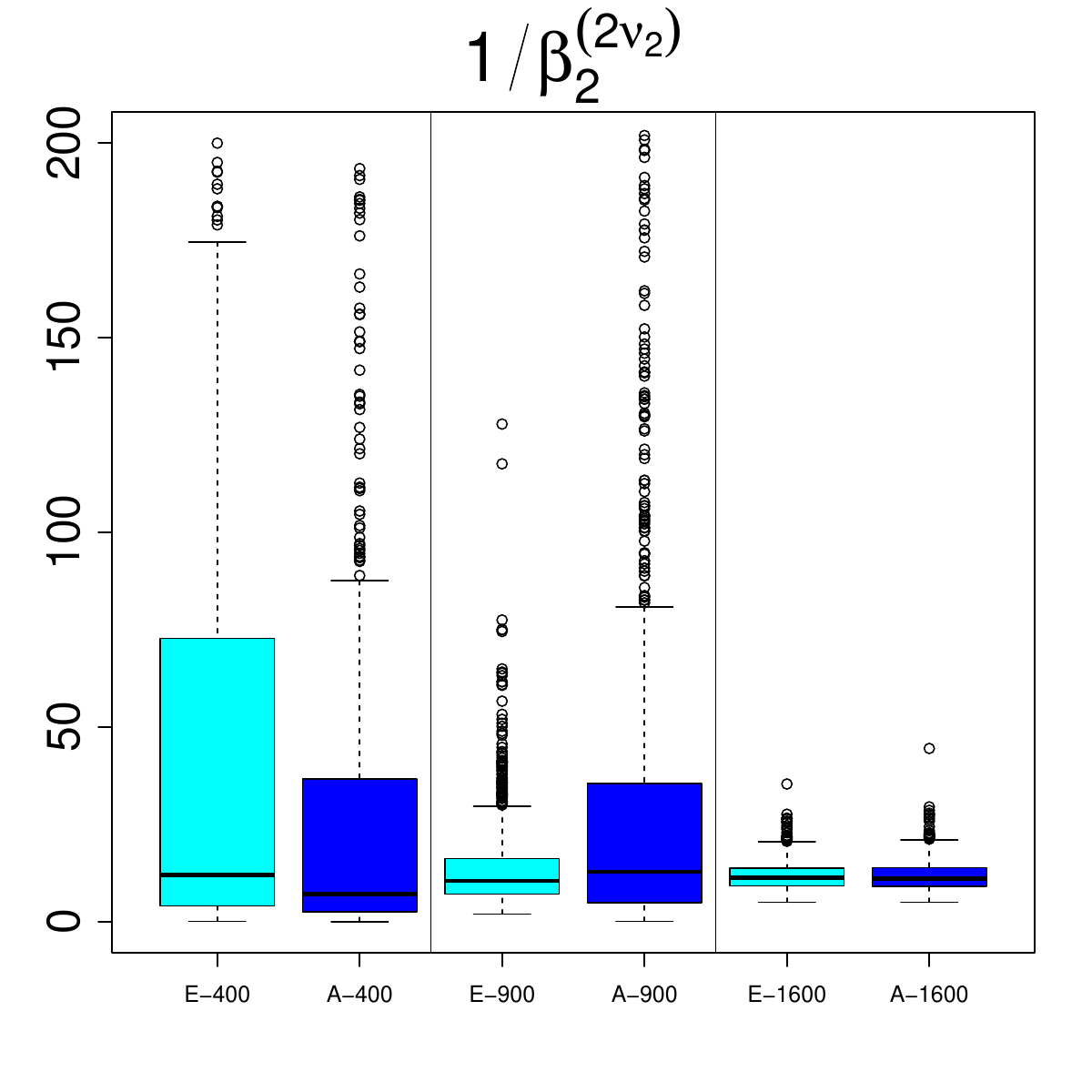}
\end{subfigure}
\caption{The empirical densities of $\hat{\bTheta}_{(\bz_n,\bbeta)}^k$ (marked in Cyan and denoted as E-$n$) and $\hat{\bTheta}_{(\Tilde{\bz}_n,\bbeta)}^k$ (marked in Blue and denoted as A-$n$) obtained from the proposed neural Bayes estimator with  $N = 500$ replicates of $n = 400,900,1600$ realizations of the GSUN process in $[0,1]^2$ simulated from $\bTheta = (2, 0.03, 0.5, 0.3, 1, -0.7, 0.5)^\top$ and Algorithm \ref{alg:uc}, respectively.}
\label{uc}
\end{figure}

To validate our proposed algorithm, we generate one replicate of $\bZ_n$ with $\bTheta = (2, 0.03, 0.5, 0.3, 1, \\  -0.7, 0.5)^\top$ with $n = 400,900,1600$ and, then, apply Algorithm \ref{alg:uc}. To demonstrate the accuracy of our uncertainty quantification method, we plot the empirical distributions of $\hat{\bTheta}_{(\Tilde{\bz}_n,\bbeta)}^k$ for $k = 500$ with a bootstrap sample size $j = 1000$, together with the empirical distributions of $\hat{\bTheta}_{(\bz_n,\bbeta)}^N$, where $N = 500$, { for coverage analysis}.

Figure \ref{uc} demonstrates that the empirical distributions of $\hat{\bTheta}_{(\Tilde{\bz}_n,\bbeta)}^k$ becomes increasingly similar to the empirical distributions of $\hat{\bTheta}_{(\bz_n,\bbeta)}^N$ as sample size $n$ increases. In particular, when $n =1,600$, the two empirical distributions almost precisely overlap for all parameters except for $(\sigma^2,\beta_1,\nu_1)$. As for $(\sigma^2,\beta_1,\nu_1)$ at $n = 1,600$, there are slightly visible deviations between the two empirical distributions but they still overlap across a major bulk. Hence, the empirical quantiles of $\hat{\bTheta}_{(\Tilde{\bz}_n,\bbeta)}^{k}$ can serve as an accurate estimate for the quantiles of $\hat{\bTheta}_{(\bz_n,\bbeta)}^N$. In addition, in our simulation study, $j$ does not play a significant role when $j > 500$. However, if $j < 500$, the empirical distributions of $\hat{\bTheta}_{(\Tilde{\bz}_n,\bbeta)}^k$ can considerably deviate from the empirical distributions of $\hat{\bTheta}_{(\bz_n,\bbeta)}^N$ due to potential biases. Hence, it would be safer to set $j$ to be a large number. 

\subsection{Comparative Study with Gaussian and Tukey \textit{g}-and-\textit{h} Processes}
In this section, we conduct a comparative study on the GSUN random field in contrast to the Gaussian and Tukey $g$-and-$h$ random fields \citep{xu2017tukey} regarding the probability integral transformations (PIT) to demonstrate its practicality.

Gaussian random fields are commonly used for modeling spatial data and dependence structures \citep{haran2011gaussian}. They rely on a trend structure and a valid covariance function, such as the Matérn covariance family, where various parameterizations are discussed in \cite{wang2023parameterization}. Their practicality has been repeatedly verified in real-life applications. In this regard, it is consequential to note that the Gaussianity assumption has intrinsic limitations in terms of symmetry and tail weight, which may not hold for many realistic datasets that exhibit skewness and heavy-tail properties, such as precipitation \citep{mondal2023tile}.

To address these concerns, \cite{xu2017tukey} proposed Tukey $g$-and-$h$ random fields, defined as follows:
$$
Y(\mathbf{s}) = \xi + \mathbf{X}(\mathbf{s})^\top\boldsymbol{\beta} + \omega T(\mathbf{s}),
$$
where $\xi \in \mathbb{R}$ is the shift parameter, $\omega > 0$ is a scale parameter, and $\mathbf{X}(\mathbf{s}),\boldsymbol{\beta} \in \mathbb{R}^p$ are the covariate vector observed at location $\mathbf{s}$ and its regression coefficient. Here, $T(\mathbf{s}) = \tau_{g,h}\{Z(\mathbf{s})\}$ and $\tau_{g,h}(z) = g^{-1}(\text{exp}(gz) -  1)\text{exp}(hz^2/2)$ is the Tukey's $g$-and-$h$ transformation introduced in 
{\setlength{\parindent}{0cm}
\begin{figure}[H]
\centering
\begin{subfigure}{0.25\textwidth}
  \centering
  \includegraphics[width=1\textwidth,]{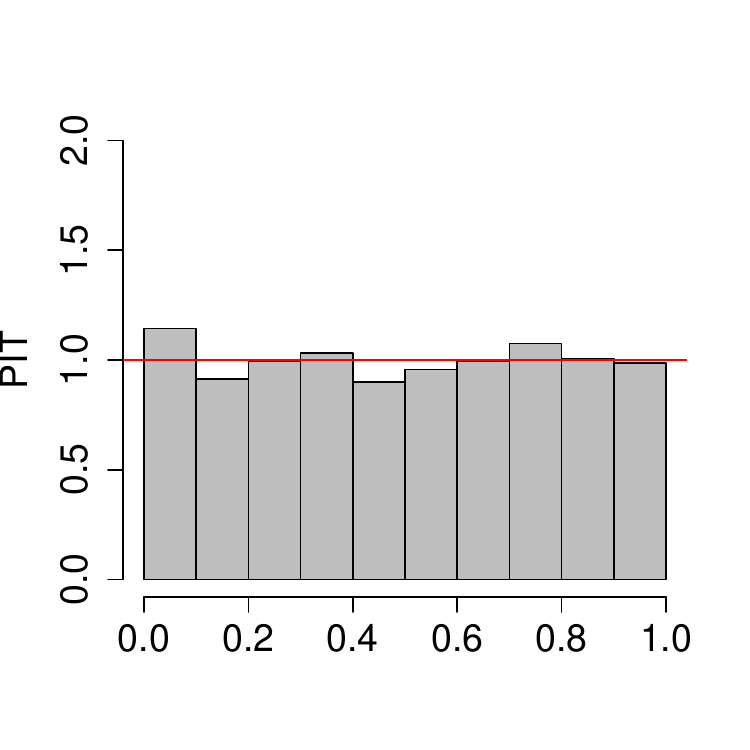}
  \caption{GSUN - GSUN.}
  \label{pit_1}
\end{subfigure}
\begin{subfigure}{0.25\textwidth}
  \centering
  \includegraphics[width=1\textwidth,]{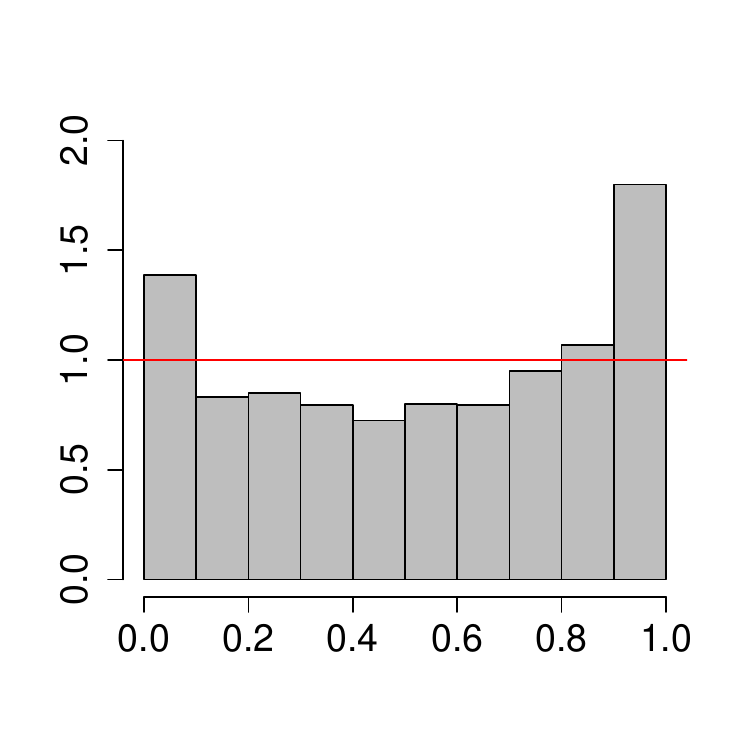}
  \caption{Gaussian - GSUN.}
  \label{pit_2}
\end{subfigure}
\begin{subfigure}{0.25\textwidth}
    \centering
    \includegraphics[width = 1\textwidth,]{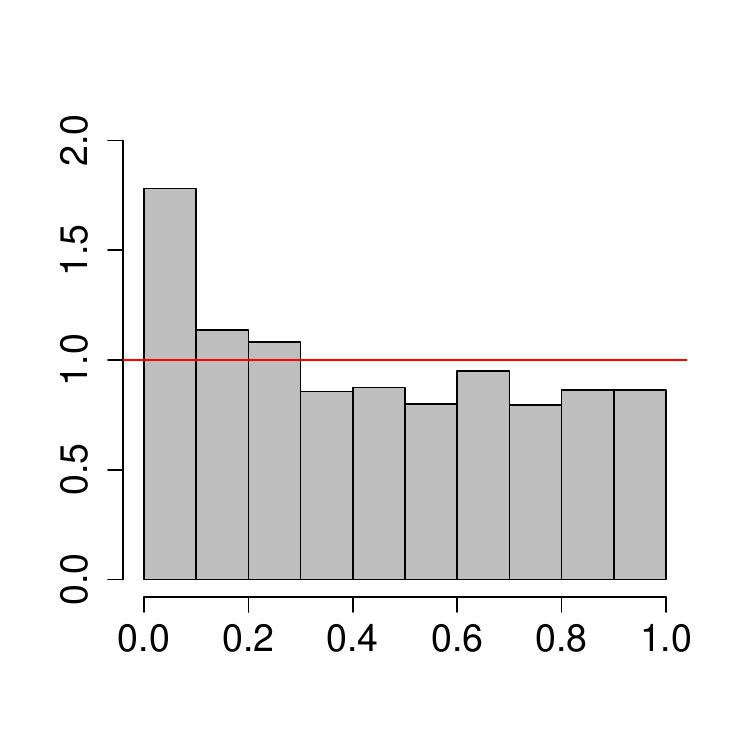}
    \caption{TGH - GSUN.}
    \label{pit_3}
\end{subfigure}
\begin{subfigure}{0.25\textwidth}
    \centering
    \includegraphics[width = 1\textwidth,]{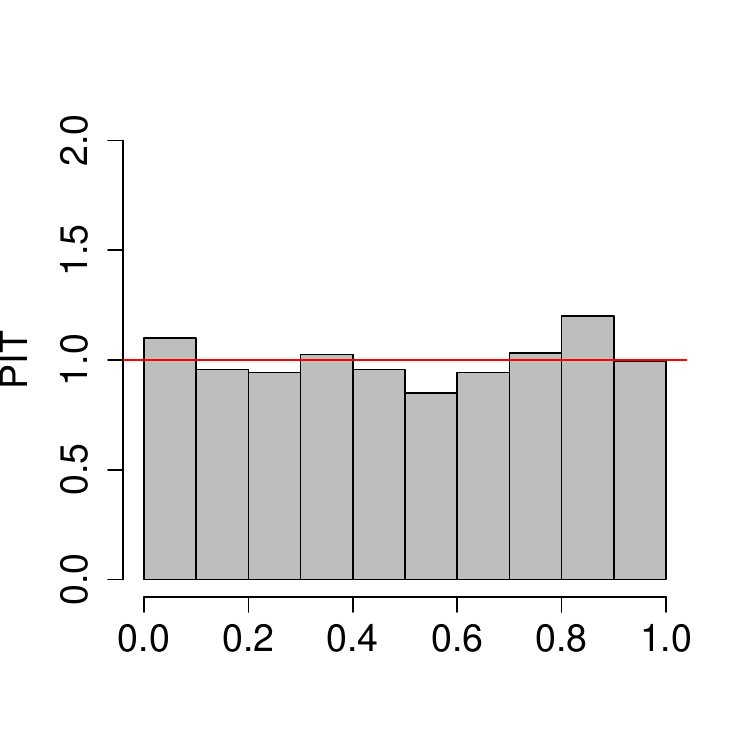}
    \caption{GSUN - Gaussian.}
    \label{pit_4}
\end{subfigure}
\begin{subfigure}{0.25\textwidth}
    \centering
    \includegraphics[width = 1\textwidth,]{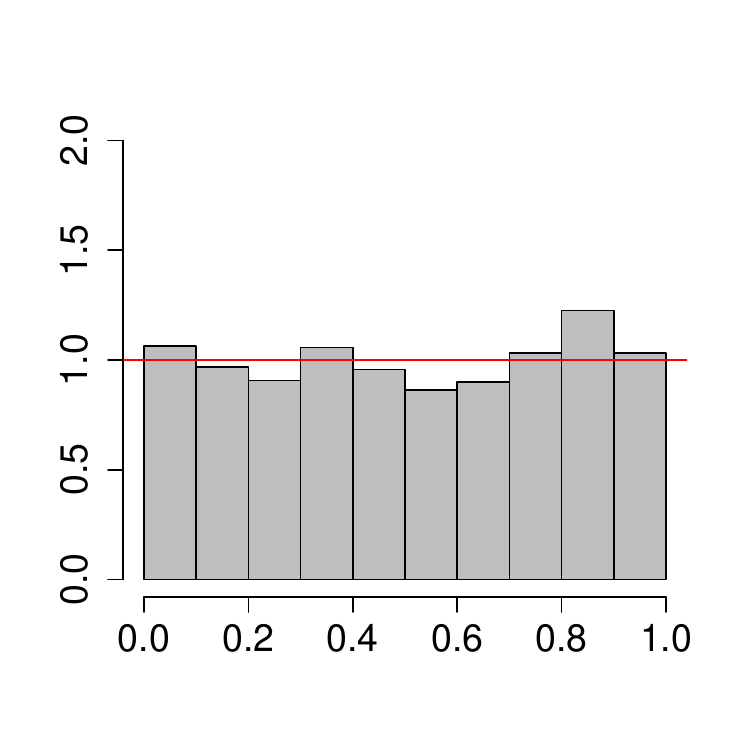}
    \caption{Gaussian - Gaussian.}
    \label{pit_5}
\end{subfigure}
\begin{subfigure}{0.25\textwidth}
    \centering
    \includegraphics[width = 1\textwidth,]{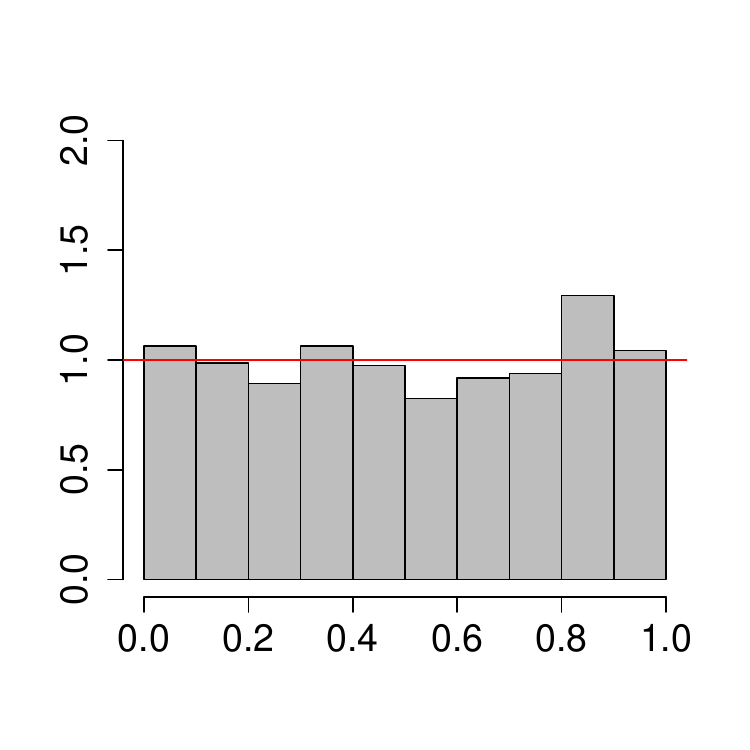}
    \caption{TGH - Gaussian.}
    \label{pit_6}
\end{subfigure}
\begin{subfigure}{0.25\textwidth}
    \centering
    \includegraphics[width = 1\textwidth,]{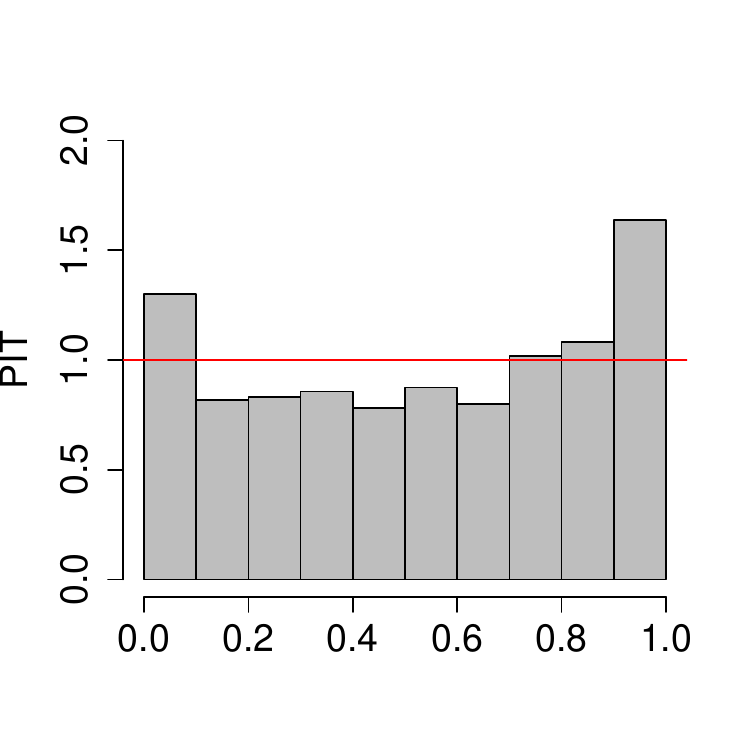}
    \caption{GSUN - TG.}
    \label{pit_7}
\end{subfigure}
\begin{subfigure}{0.25\textwidth}
    \centering
    \includegraphics[width = 1\textwidth,]{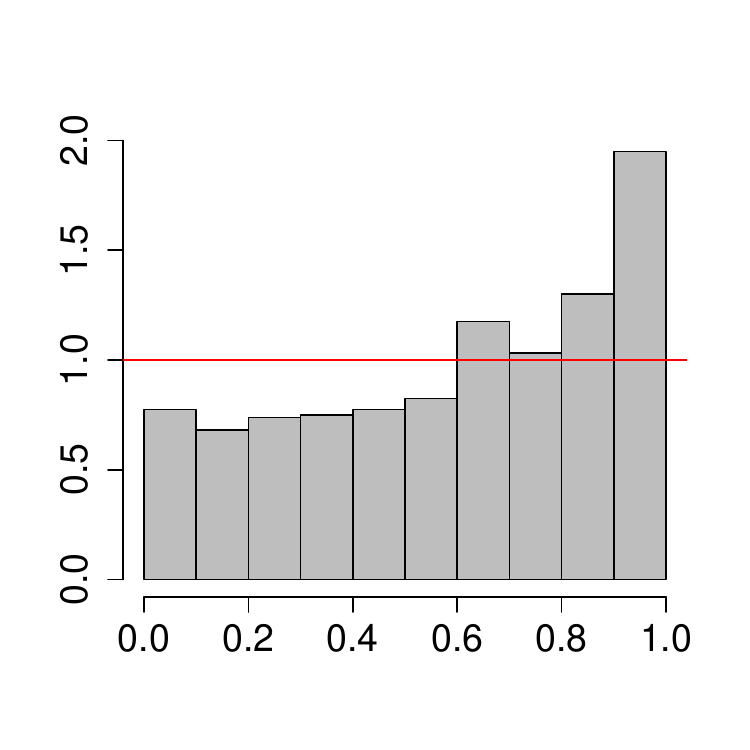}
    \caption{Gaussian - TG.}
    \label{pit_8}
\end{subfigure}
\begin{subfigure}{0.25\textwidth}
    \centering
    \includegraphics[width = 1\textwidth,]{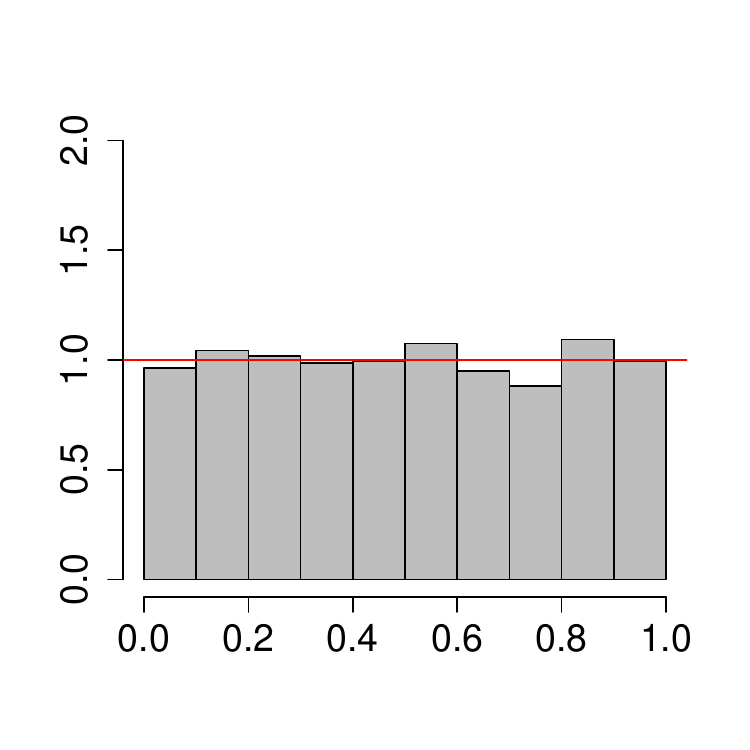}
    \caption{TG - TG.}
    \label{pit_9}
\end{subfigure}
\begin{subfigure}{0.25\textwidth}
    \centering
    \includegraphics[width = 1\textwidth,]{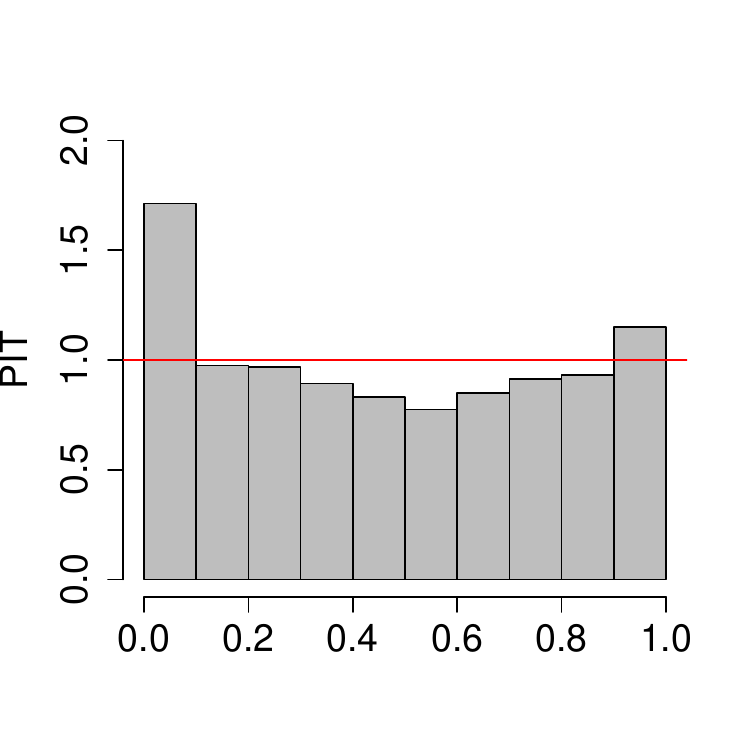}
    \caption{GSUN - TGH.}
    \label{pit_10}
\end{subfigure}
\begin{subfigure}{0.25\textwidth}
    \centering
    \includegraphics[width = 1\textwidth,]{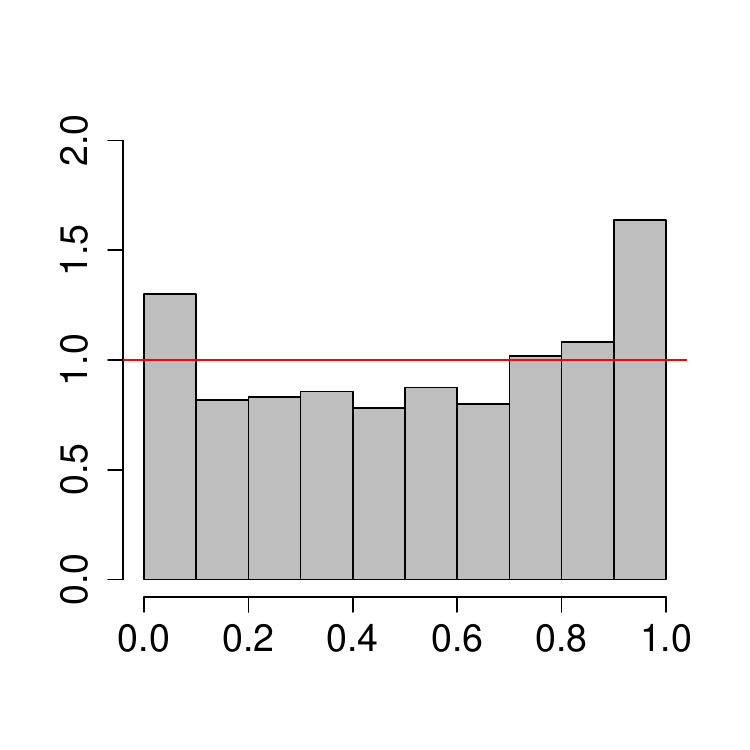}
    \caption{Gaussian - TGH.}
    \label{pit_11}
\end{subfigure}
\begin{subfigure}{0.25\textwidth}
    \centering
    \includegraphics[width = 1\textwidth,]{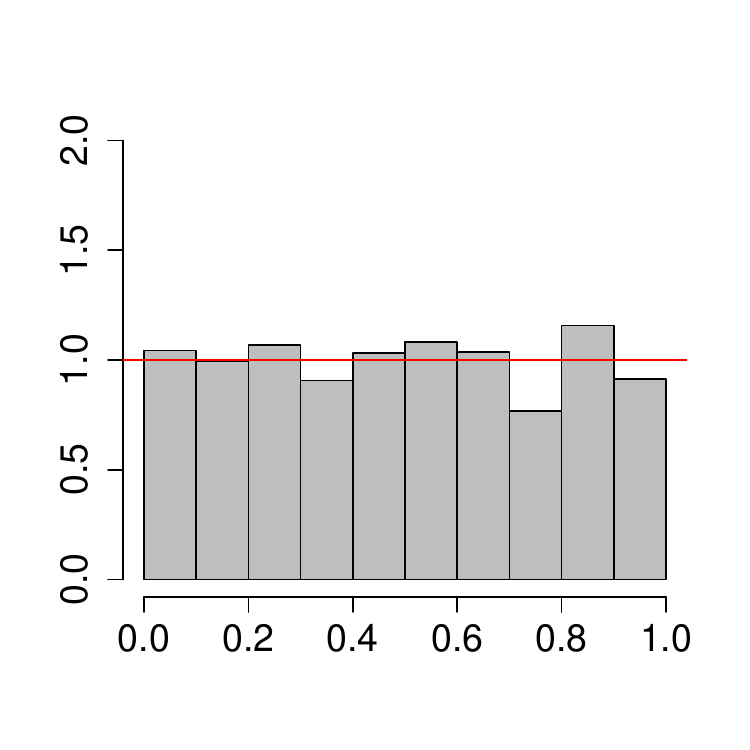}
    \caption{TGH - TGH.}
    \label{pit_12}
\end{subfigure}
\caption{The PIT results for a sample realization of the GSUN, Gaussian, Tukey $g$ (TG), and Tukey $g$-and-$h$ (TGH) processes simulated from $\bTheta = (2, 0.03, 0.5, 0.01, 1, 1, −3)^\top$ for the GSUN, $\btheta_g = (\sigma^2,\beta,\nu)^\top = (2,0.03,0.5)^\top$ for the Gaussian, $\btheta_{tg} = (\sigma^2,\beta,\nu,g)^\top = (2,0.03,0.5, 0.3)^\top$, and $\btheta_{tgh} = (\sigma^2,\beta,\nu,g,h)^\top = (2,0.03,0.5, 0.3,0.4)^\top$ for the TG and TGH at $n = 1600$ randomly generated locations on $[0,1]^2$ using the GSUN, Gaussian, TG, and TGH models. The caption, Model - Model, indicates that we apply the cdf of the former model to the realization generated using the latter.}
\label{PIT}
\end{figure}
\cite{tukey1977exploratory}, with $Z(\bs)$ as a Gaussian process. In this case, $g$ dictates the amount and direction of skewness (to the right if $g > 0$ and to the left if $g < 0$) and $h$ governs the tail behavior (where a larger $h$ leads to heavier tail weights). Lastly, if a random vector $\bZ \sim {\cal N}_n(\bmu_n,\bSigma_n)$, 

$\mathbf{T} = \tau_{g,h}(\bZ) \sim {\cal GH}_n(\bmu_n,\bSigma_n,g,h)$, where ${\cal GH}_n(\bmu_n,\bSigma_n,g,h)$ denotes a multivariate Tukey $g$-and-$h$ distribution with shift $\bmu_n$, dispersion matrix $\bSigma_n$, and $g,h$ for skewness and tail-heaviness.}

To apply the Probability Integral Transform (PIT), we need to calculate the cumulative distribution functions (cdfs) for each case. The cdf for the Gaussian distribution is well-known and used regularly. The cdf for the GSUN random field can be derived easily using Proposition~\ref{p6} in the Supplementary Materials. For the TGH 
random field, we can apply the inverse Tukey's $g$-and-$h$ transformation $\tau_{g,h}^{-1}\{Y(\bs)\}$ to obtain the underlying Gaussian field $Z(\bs)$ and then use the Gaussian cdf for PIT.

Moreover, we can use the Maximum Likelihood Estimation (MLE)-based inference  
as implemented in \cite{abdulah2018exageostat} and \cite{mondal2023tile} for parameter estimation of the Gaussian and TGH spatial models.

Figure \ref{PIT} demonstrates the PIT results under three different spatial models. In particular, Figure~\ref{pit_1} shows that the PIT is quite uniformly distributed compared to the results shown in Figures~\ref{pit_2} and \ref{pit_3}, which contain obvious concentrations either at the center or tail. Hence, the GSUN spatial process has its { own distinct asymmetry and tail-weight differing} from Gaussian and TGH fields.  

The same conclusion can also be made in Figures~\ref{pit_7} and \ref{pit_10}, where we can see a striking difference between TGH and GSUN. Figures~\ref{pit_4} and \ref{pit_6} demonstrate that both GSUN and TGH involve the Gaussian process as a special case. In addition, Figures~\ref{pit_3} and \ref{pit_7} indicate that the GSUN and TG typically have differing mean and tail weight; TG has a heavier tail weight even if there is only skewness ($h = 0$) in the model. Figure~\ref{pit_10} indicates that when $h$ is rather large, TGH possesses much heavier tails than the GSUN, further highlighting the difference between TGH and GSUN.

\section{Application to Pb-contaminated Soil Data} \label{real_data}
In this section, we apply the GSUN to the same real data as in \cite{zareifard2013non}, which contains 117 observations of Pb-contaminated areas in soils of a region of north Iran. In detail, we fit the GSUN and the SUGLG model proposed in \cite{zareifard2013non} to obtain the parameter estimates and then, plot the PIT results with the fitted models.  
\begin{figure}[b!]
\centering
\begin{subfigure}{0.44\textwidth}
  \centering
  \includegraphics[width=1\textwidth,]{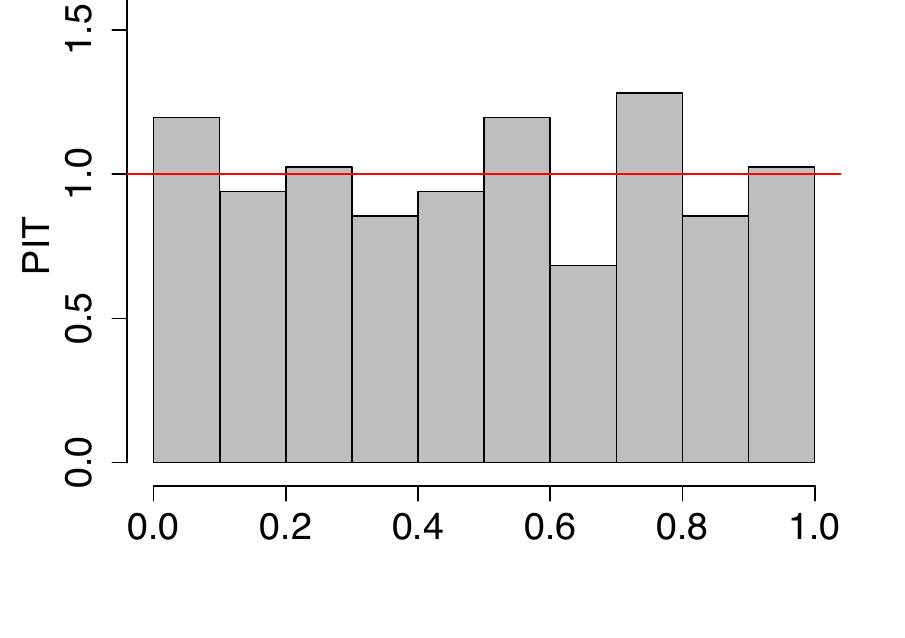}
  \vspace{-10mm}
\caption*{GSUN}
\end{subfigure}
\begin{subfigure}{0.44\textwidth}
  \centering
  \includegraphics[width=1\textwidth,]{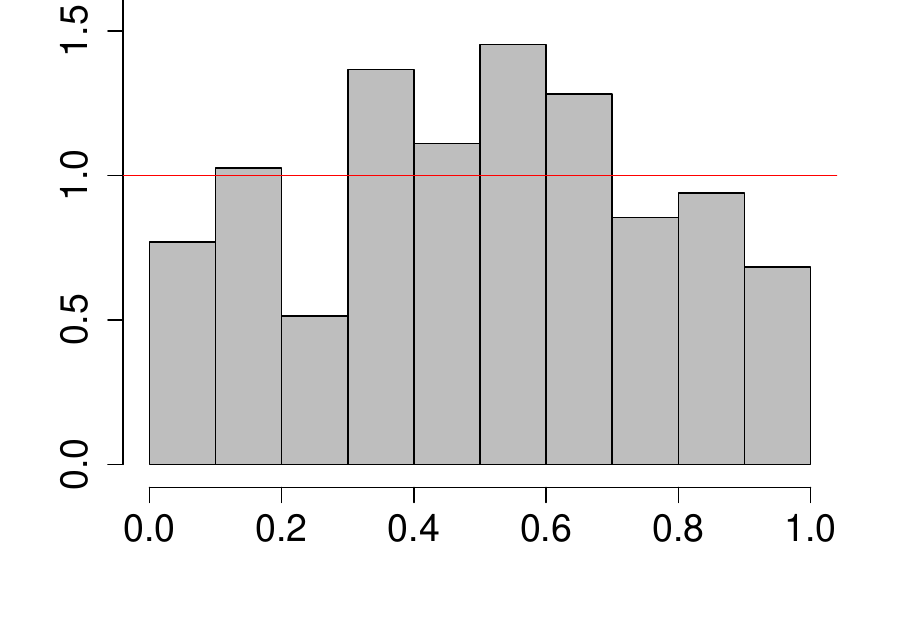}
  \vspace{-10mm}
  \caption*{SUGLG}
\end{subfigure}
\caption{The PIT results for the 117 observations of Pb-contaminated areas in soils of a region of north Iran using GSUN and SUGLG spatial models.}
\label{rPIT}
\end{figure}

Figure~\ref{rPIT} shows that the PIT with the GSUN model is closer to a uniform distribution compared to the PIT with SUGLG. Such results are anticipated because the GSUN model allows for more flexibility in the latent process that 
involves its own range and smoothness parameters, whereas the SUGLG model assumes the same covariance structure for the observed and latent processes.  In addition, GSUN uses one additional parameter to control the skewness compared to SUGLG and the skewness matrix $\bH$ in GSUN is an aggregation of the principal components of the covariance matrices of the observed and latent processes (the skewness matrix is set as a scalar multiple of the correlation matrix of the observed process in SUGLG), carrying more statistical intuition and interpretability and therefore, allowing for a better fit of the data. 

\section{Conclusion} \label{conclude}
In this work, we used a more convenient parameterization of the SUN distribution. Here, $\bH$ represents the direction of skewness, providing a straightforward representation. This parameterization also helps to avoid the numerical instability issue from the original parametrization. Additionally, we proposed a GSUN spatial model derived from a concise re-parameterization. This model ensures { vanishing correlations in large distances} and offers a more statistically interpretable skewness matrix. The GSUN process is more general compared with the conventional Gaussian process model. It includes the Gaussian process as a special case and models the range and smoothness of the latent process that governs the skewness, in contrast to keeping them fixed as in \cite{zareifard2013non}.

Due to the complexity of the parameter inference of the SUN distribution, we adopted a neural Bayes estimator to obtain point summaries of the parameters. This was achieved using a GAT- and Encoder-structured neural network while minimizing the Bayes risk. The neural Bayes estimator can be time-consuming in the training process and requires millions of simulations. However, once sufficiently trained, the inference stage is quite efficient. To this end, the proposed neural Bayes estimator in this work outperforms the conventional CNN-based architectures in terms of accuracy and stability, indicating the more enhanced modeling capacity of GAT and Encoder together with the graphical representation of spatial data. 

Moreover, we compared the GSUN model with Gaussian and Tukey $g$-and-$h$ models,  two popularly used spatial models, to demonstrate its uniqueness through PIT plots. We applied the GSUN spatial process to 
117 observations of Pb-contaminated areas in soils of a region of north Iran and showcased its better fit over the SUGLG process.

\section*{Acknowledgement}
This publication is based upon work supported by King Abdullah University of Science and Technology Research Funding (KRF) under Award No. ORFS-2022-CRG11-5069.

\newpage
\section*{Supplementary Materials}

\setcounter{figure}{0}
\setcounter{table}{0}
\renewcommand{\thefigure}{S\arabic{figure}}
\renewcommand{\thetable}{S\arabic{table}}
\subsection*{S.1 Proof of Equivalence for the Re-parameterization of the SUN}
\begin{proof}
 Given Definition \ref{defSUN}, here is the detailed proof:
 \begin{align*}
   (\bY|\bU=\bu)&\myeq\bxi+\bH\bu+\bW,  
 \end{align*}
 which is a normal random vector with $\text{E}(\bY|\bU=\bu)=\bxi+\bH\bu$ and $\text{cov}(\bY|\bU=\bu)=\bPsi$. Next, by Definition \ref{defSUN}, we have $\bU=(\bW_{0}|\bW_{0} + \btau > \0)\sim{\cal TN}_{q}(-\btau; \0,\Bar{\bGamma})$, where ${\cal TN}_q(-\btau;\0,\bar\bGamma)$ denotes a $q$-dimensional truncated multivariate normal from below $-\btau$ with mean $\0$ and covariance $\bar\bGamma$. Therefore, the pdf of $(\bY|\bU=\bu)$ is
\begin{align*}
    f(\by|\bu)=\phi_p\left(\by;\bxi+\bH\bu,\bPsi\right).
\end{align*}
We can also derive the pdf of $\bU$ as 
\begin{align*}
    g(\bu) 
    &= \frac{\Prob(\bU=\bu,\bU + \btau > \0)}{\Prob(-\bU  < \btau)}=\frac{\phi_{m}(\bu;\0,\Bar{\bGamma})}{\Phi_{m}(\btau;\0,\Bar{\bGamma})}=\frac{\phi_{m}(\bu;\0,\Bar{\bGamma})}{\Phi_{m}(\btau;\0,\Bar{\bGamma})},\quad \bu + \btau > \0.
\end{align*}
Therefore, for the pdf of $\bY$ we have that
 \begin{align*}
     f(\by)&=\int_{\bu + \btau >\0}f(\by|\bu)g(\bu)\text{d}\bu\\
           &=\frac{1}{\Phi_{m}(\btau;\0,\Bar{\bGamma})}\int_{\bu + \btau > \0}\phi_d\left(\by;\bxi+\bH\bu,\bPsi\right)\phi_m(\bu;\0,\Bar{\bGamma})\text{d}\bu.
 \end{align*}
Here the marginal-conditional representation of a $(d+m)$-variate normal pdf yields the following identity:
\begin{align*}
    \phi_{d}\left(\by;\bxi+\bH\bu,\bPsi\right)\phi_q(\bu;\0,\bar\bGamma)=
    \phi_d\left(\by;\bxi,\bPsi+\bH\bGamma\bH^\top\right)\phi_m(\bu;\bmu,\bSigma),
\end{align*}
where
\begin{align*}
&\bmu = \Bar{\bGamma}\bH^\top\left(\bPsi+ \bH\Bar{\bGamma}\bH^\top\right)^{-1}\left(\by-\bxi\right), \\
&\bSigma=\Bar{\bGamma}-\Bar{\bGamma}\bH^\top\left(\bPsi+\bH\Bar{\bGamma}\bH^\top\right)^{-1}\bH\Bar{\bGamma}.
\end{align*}
Following from the previous result, 
\begin{align*}
f(\by)&=\frac{\phi_d\left(\by;\bxi,\bPsi+\bH\Bar{\bGamma}\bH^\top\right)}{\Phi_{m}(\btau;\0,\Bar{\bGamma})}\int_{\bu + \btau > \0}\phi_m(\bu;\bmu,\bSigma)\text{d}\bu\\
    &=\frac{\phi_p\left(\by;\bxi,\bPsi+\bH\Bar{\bGamma}\bH^\top\right)}{\Phi_{m}(\btau;\0,\Bar{\bGamma})}\int_{\textbf{v}+\bmu + \btau > \0}\phi_m(\textbf{v};\0,\bSigma)\text{d}\textbf{v},
\end{align*}
where by symmetry
\begin{align*}
\int_{\textbf{v}+\bmu + \btau >\0}\phi_m(\textbf{v};\0,\bSigma)\text{d}\textbf{v}=\int_{\textbf{v}>-\bmu - \btau}\phi_m(\textbf{v};\0,\bSigma)\text{d}\textbf{v}=\int_{\textbf{v}<\bmu + \btau}\phi_m(\textbf{v};\0,\bSigma)\text{d}\textbf{v}=\Phi_m(\bmu;\0,\bSigma).
\end{align*}
 Putting all the pieces together we have 
\begin{align*}
f(\by)=\phi_p\left(\by;\bxi,\bPsi+\bH\Bar{\bGamma}\bH^\top\right)\frac{\Phi_{m}(\bmu+ \btau;\0,\bSigma)}{\Phi_{m}(\btau;\0,\Bar{\bGamma})},
\end{align*}
which matches exactly with the pdf of the SUN random vector calculated in \cite{arellano2006unification} if we replace the parameters accordingly.  
\end{proof}
We also present the selection representation \citep{arellano2006unification} of the SUN for \(\bY\). 
\begin{prop} \label{select}
    If $\bY \sim {\cal SUN}_{d,m}(\bxi,\bPsi,\bH,\btau,\Bar{\bGamma})$, then 
\begin{align*}
    \bY \myeq (\bW_*|\bW_0 + \btau > \0), \text{with} \begin{pmatrix}
    \bW_*\\
    \bW_0
    \end{pmatrix}\sim {\cal N}_{d+m}\left\{\begin{pmatrix}
    \bxi\\
    \0
    \end{pmatrix},\begin{pmatrix}
    \bPsi + \bH\bar\bGamma\bH^\top & \bH\bar\bGamma\\
    \bar\bGamma\bH^\top & \bar\bGamma
    \end{pmatrix}\right\}.
\end{align*}
\end{prop}
\begin{proof}
We prove this by the equivalence of pdf. In particular, we compute the pdf of $(\bW_*|\bW_0 + \btau > \0)$ and have that 
\begin{align*}
    f(\by) & =\frac{\Prob(\bW_0 + \btau > \0|\bW_*=\by)f_{\bW_*}(\by)}{\Prob(\bW_0 > \0)}=\frac{\Prob(-\bW_0 < \btau|\bW_*=\by)f_{\bW_*}(\by)}{\Prob(-\bW_{0} < \btau)}\\
    &=\frac{\Prob(-\bW_{0} + \bmu < \bmu + \btau|\bW_*=\by)f_{\bW_*}(\by)}{\Prob(-\bW_{0} < \btau)}\\
    &=\phi_d\left(\by;\bxi,\bPsi+\bH\Bar{\bGamma}\bH^\top\right) \frac{\Phi_{m}(\bmu+\btau;\0,\bSigma)}{\Phi_{m}(\btau;\0,\bar\bGamma)},
\end{align*}
which is identical to the pdf derived from Definition \ref{defSUN}.
\end{proof}
\subsection*{S.3 Vanishing correlation of Non-negative Gaussian Field}
\begin{prop} \label{trunc_ind}
    Let $\bW^+ \sim {\cal TN}_n(\0;\0,\textbf{D})$, where $\textbf{D} = \diag(D_1, \dots, D_n)$. Then, the marginal random variables of $\bW^+$ are independent.
\end{prop}
\begin{proof}
    Given that $\bW^+ \sim {\cal TN}_n(\0;\0,\textbf{D})$, we have the pdf of $\bW^+$ as 
    \begin{align*}
    f_{\bW^+}(\bw^+) & = \frac{\phi_n(\bw^+;\0,\textbf{D})}{\Phi_n(\0;\0,\textbf{D})} = \prod_{i = 1}^n \frac{\phi(w^+_i;0,D_i)}{\Phi(0;0,D_i)} = \prod_{i = 1}^n f_{W^+_i}(w^+_i),
    \end{align*}
where $W^+_i \sim {\cal TN}(0;0,D_i)$.
\end{proof}
\subsection*{S.4 CDF of the Re-parameterized SUN distribution}
\begin{prop}\label{p5}
    If $\bX \sim {\cal SUN}_{p,q}(\bxi,\bPsi,\bH,\btau,\bar\bGamma)$, then 
\begin{align*}
    \bX\myeq(\bW_*|\bW_0 + \btau > \0), \text{with} \begin{pmatrix}
    \bW_*\\
    \bW_0
    \end{pmatrix}\sim {\cal N}_{p+q}\left\{\begin{pmatrix}
    \bxi\\
    \0
    \end{pmatrix},\begin{pmatrix}
    \bPsi + \bH\bar\bGamma\bH^\top & \bH\bar\bGamma\\
    \bar\bGamma\bH^\top & \bar\bGamma
    \end{pmatrix}\right\}. 
\end{align*}
\end{prop}
\begin{proof}
We prove this by the equivalence of mpdf.  We compute the mpdf of $(\bW_*|\bW_0 > \0)$ and we have that 
\begin{align*}
    f(\bx) & =\frac{\Prob(\bW_0 + \btau > \0|\bW_*=\bx)f_{\bW_*}(\bx)}{\Prob(\bW_0 + \btau > \0)} =\frac{\Prob(-\bW_{0} + \bmu < \bmu + \btau|\bW_*=\bx)f_{\bW_*}(\bx)}{\Prob(-\bW_{0}+\bmu < \bmu + \btau)}\\
&=\phi_p\left(\bx;\bxi,\bPsi+\bH\bar\bGamma\bH^\top\right) \frac{\Phi_{q}(\bmu+\btau;\0,\bSigma)}{\Phi_{q}(\btau;\0,\bar\bGamma)}.
\end{align*}
Notice that the result computed here agrees with the mpdf of the SUN distribution computed by definition. Hence, $\bX\myeq(\bW_*|\bW_0 + \btau > \0)$.
\end{proof}
\begin{prop}\label{p6}
    The cdf of \(\bX \sim {\cal SUN}_{p,q}(\bxi,\bPsi,\bH,\btau,\bar\bGamma)\) is 
\begin{align*}
    F(\bx)=\frac{\Phi_{p+q}(\bx_*;\bxi_*,\bOmega_*)}{\Phi_{q}(\btau;\0,\bar\bGamma)},\quad \bx\in \R^p,
\end{align*}
where $\bx_*=\left(\bx^\top,\btau^\top\right)^\top,\bxi_*=(\bxi^\top,\0^\top)^\top$, and $\bOmega_* = \begin{pmatrix}
\bPsi+\bH\bar\bGamma\bH^\top & -\bH\bar\bGamma\\
-\bar\bGamma\bH^\top & \bar\bGamma
\end{pmatrix}.$    
\end{prop}
\begin{proof}
 By Proposition \ref{p5}, we have \(\bX\myeq(\bW_*|\bW_0 + \btau > \0)\) which is equivalent to \((\bW_*|-\bW_0 < \btau)=(\bW_*|\bY_0 < \btau)\) where \(\bY_0=-\bW_0 \sim{\cal N}(\0,\bar\bGamma)\). Then,
$\begin{pmatrix}
\bW_*\\
\bY_0
\end{pmatrix} 
\sim {\cal N}_{p+q}(\bxi_*,\bOmega_*)$ and 
\begin{align*}
    F(\bx)=\Prob(\bX\leq \bx)=\Prob(\bW_* \leq \bx|\bY_0 < \btau)=\frac{\Prob(\bW_* \leq \bx,\bY_0 < \btau)}{\Prob(\bY_0 < \btau)}.
\end{align*}
\end{proof}
\subsection*{S.4 CNN-based Neural Bayes Estimator}
\begin{figure}[H]
  \centering
  \includegraphics[width=0.8\linewidth]{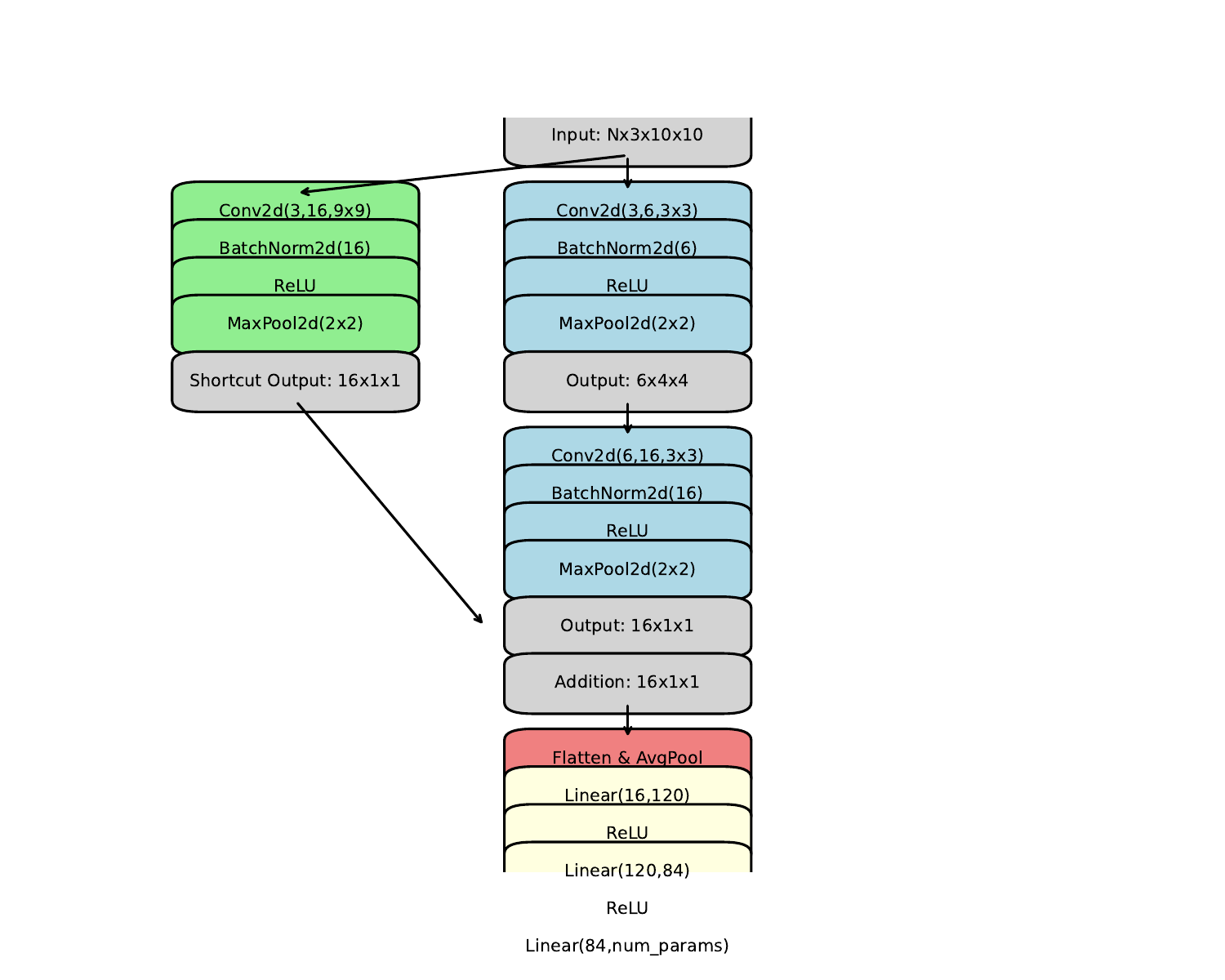}
\caption{CNN-based Architecture for the Neural Bayes Estimator.}
\label{arc_neural_bayes}
\end{figure}
First, we have $\bZ \sim {\cal SUN}_{n,n}(\0,\bSigma(\btheta_1),\bH,\0,\bSigma(\btheta_2)) \in \R^{n \times n}$. Next, we transform the set of locations $\bs = (s_1, s_2)^\top$ into $\R^{2\times n \times n}$ (putting $s_1^i$ and $s_2^i$, where $i = 1,\dots, n$, on two $n$-by-$n$ grids, respectively and stacking them together) and stack the dimensionally transformed $\bs$ with $\bZ$ to create a three-channel dataset, in symbols $\bZ \| \bs \in \R^{3 \times n \times n}$. This new dataset is similar to RGB image data, where instead of colors, spatial realizations are grouped with their locations across different channels. With this approach, a multi-channel CNN can effectively capture the dependency structure among the observed samples alongside their spatial information. Furthermore, this reformulation can be widely applied to both regularly and irregularly spaced data.

The diagram in Figure \ref{arc_neural_bayes} shows our proposed architecture of the Neural Bayes Estimator. The network comprises three CNN blocks for information aggregation and one DNN block for prediction. The blue CNN blocks represent information aggregation on the forward path. The green CNN block denotes the shortcut or residual block, which is a commonly used technique proposed in \cite{targ2016resnet} to address the vanishing gradient problem \citep{lau2018review} during training. Additionally, for each CNN block, 2-D batch normalization \citep{bjorck2018understanding} is applied to ensure identical scaling, along with the rectified linear unit (ReLU) activation function and 2-D max pooling \citep{nagi2011max} to filter out trivial features. In the final block, the network performs average pooling \citep{gholamalinezhad2020pooling} across the batch dimension to summarize the features learned in each replicate. Similar to the CNN blocks, ReLU activation functions ensure that only significant features pass through. The two linear layers in this block function to map distinct features to point summaries of $\bTheta$. 

The number of CNN layers, kernel sizes, and neurons for the linear layers are tuned according to the convergence efficiency. In short, training loss of the proposed architecture in Figure \ref{arc_neural_bayes} reaches below the threshold and stabilizes at around 30 million simulations. If one block is added in the forward path, the estimator takes longer to be properly trained. In addition, if the shortcut is removed, the estimator fails to reach below the threshold; see Figure \ref{training_loss_CNN} for more details.

To train the CNN-based neural Bayes Estimator, we use the same prior distributions and settings as in Section \ref{train_simulation}. Nonetheless, we adopt a different simulation scheme for training data. In short, we simulate $N = 160$ replicates of $\bZ \sim {\cal SUN}_{n,n}(\0, \bSigma(\btheta_1),\bH,\0, \textbf{C}(\btheta_2))$ with $n =100$ and $\bTheta = (1,0.15,1,0.1,0.5,0.55,-0.3)^\top$ from randomly generated locations. Then, we stack the dataset together as one input. By doing this, we can have more replicates of smaller spatial data, which helps to stabilize the loss during training. Moreover, the number of effective data remains the same as in Section \ref{train_simulation}.

\subsection*{S.5 Training and Hyper-parameter Tuning for the Neural Bayes Estimators}
\begin{figure}[H]
  \centering
  \includegraphics[width=0.8\linewidth]{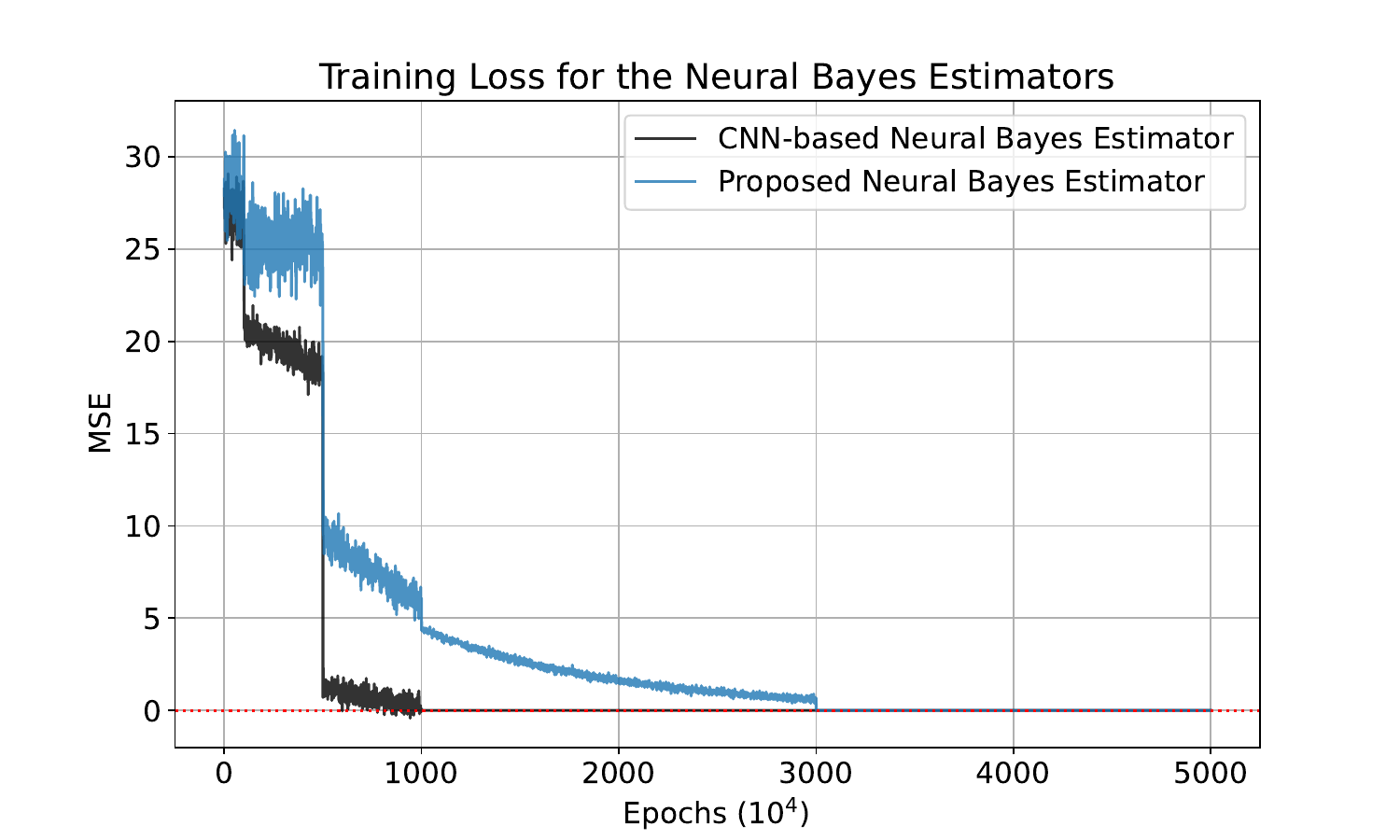}
\caption{Training loss curves for our proposed neural Bayes estimator (Blue) and the CNN-based Bayes estimator. The red line denotes the stopping criterion $10^{-5}$. The learning rates for both estimators start as $10^{-3}$ and are manually adjusted (multiplied by 0.1) at $(100,500,1000,3000) \times 10^4$ epochs.  }
\label{training_loss}
\end{figure}

\begin{figure}[H]
  \centering
  \includegraphics[width=0.8\linewidth]{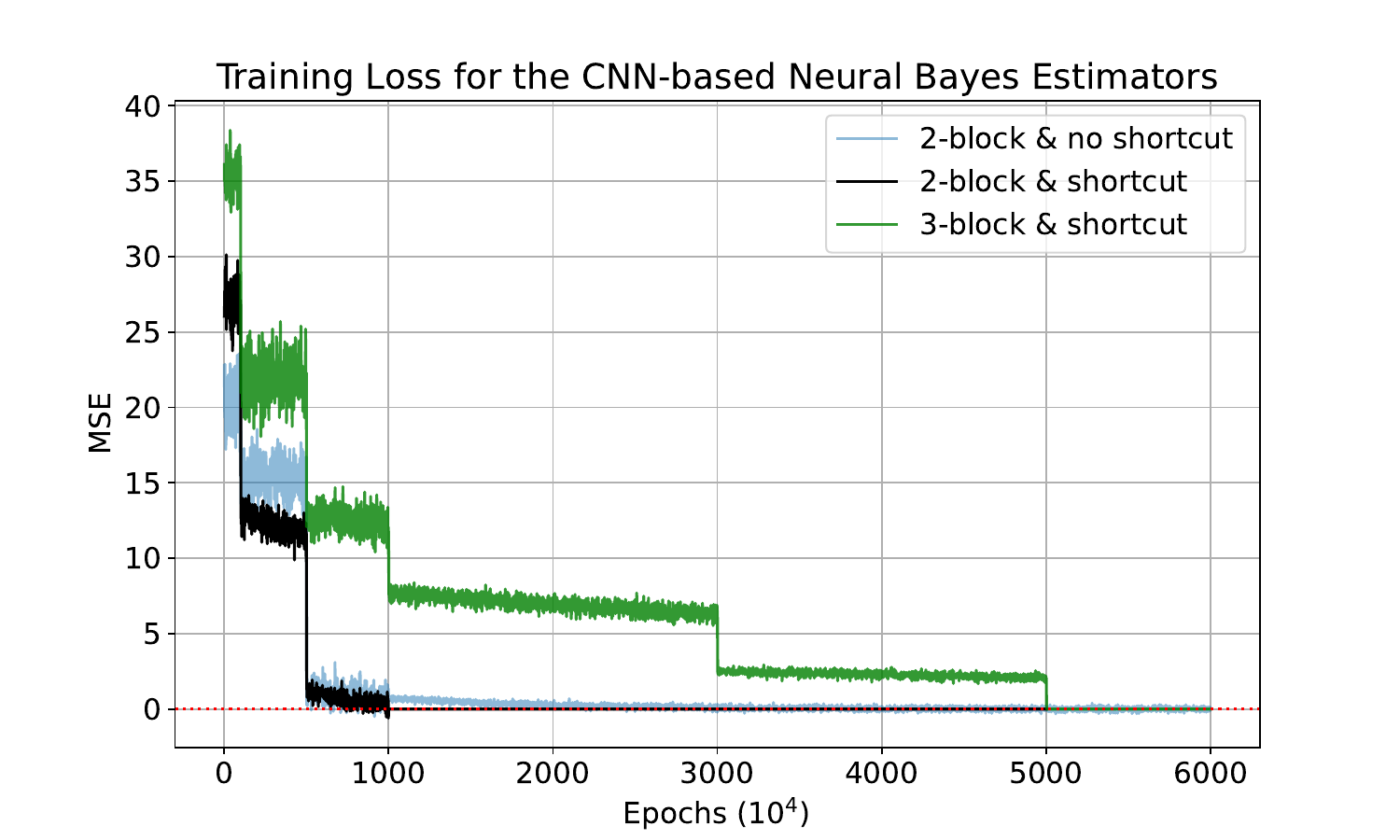}
\caption{Training loss curves for the CNN-based Bayes estimators with various hyper-parameters. The red line denotes the stopping criterion $10^{-5}$. The learning rates for both estimators start as $10^{-3}$ and are manually adjusted (multiplied by 0.1) at $(100,500,1000,3000) \times 10^4$ epochs.   }
\label{training_loss_CNN}
\end{figure}

In short, Figure~\ref{training_loss} demonstrates that our proposed neural Bayes estimator reaches below the threshold and converges at a slower rate compared with the conventional CNN-based architectures because graph and transformer architectures are more difficult to train. In addition, Figure~\ref{training_loss_trans} shows that having an encoder transformer block indeed helps the Bayes estimator to converge, which can be attributed to the enhanced modeling capacity and more profound understandings of the inherent dependence structure. Otherwise, the estimator experiences relatively large volatility even with $5\times 10^7$ simulated training data. In addition, Figures~\ref{training_loss_CNN} and \ref{training_loss_GAT} have demonstrated that having more layers of a particular architecture does not always lead to better efficiency. Redundant layers can cause the neural Bayes estimator to converge at a slower rate (more simulations needed to be properly trained). Lastly, Figure~\ref{training_loss_CNN} also illustrates that having a shortcut connection in the CNN architectures can help the Bayes estimator to reach a lower threshold (more properly trained). 
\bibliographystyle{apalike}
\bibliography{reference}

\end{document}